\documentclass[conference]{IEEEtran}
\IEEEoverridecommandlockouts

\usepackage{cite}
\usepackage{amsmath,amssymb,amsfonts}
\usepackage{algorithm}
\usepackage{algorithmic}
\usepackage{graphicx}
\usepackage{textcomp}
\usepackage{xcolor}
\def\BibTeX{{\rm B\kern-.05em{\sc i\kern-.025em b}\kern-.08em
    T\kern-.1667em\lower.7ex\hbox{E}\kern-.125emX}}

\usepackage{booktabs}
\usepackage{multirow}
\usepackage{microtype}
\usepackage{placeins}
\usepackage{seqsplit}
\usepackage[hidelinks]{hyperref}

\usepackage{framed}
\newif\ifdraftnotes
\draftnotesfalse 

\newcommand{\ours}{DP-GRPO}

\newif\ifplanners
\plannersfalse 
\definecolor{shadecolor}{RGB}{245,245,235}
\newenvironment{plannerbox}{\ifplanners\begin{shaded}\small\noindent\textbf{Planner.}\ \ignorespaces\else\setbox0\vbox\bgroup\fi}{\ifplanners\end{shaded}\else\egroup\fi}

\newtheorem{definition}{Definition}
\newtheorem{proposition}{Proposition}
\newtheorem{theorem}{Theorem}
\newtheorem{lemma}{Lemma}
\newtheorem{corollary}{Corollary}
\newtheorem{remark}{Remark}

\begin{document}

\title{Reward-Driven Learning under\\Prompt-Level Differential Privacy}

\author{\IEEEauthorblockN{Jiachen Zhao}
\IEEEauthorblockA{University of Notre Dame\\jzhao7@nd.edu}
\and
\IEEEauthorblockN{Antonia Januszewicz}
\IEEEauthorblockA{University of Notre Dame\\ajanusze@nd.edu}
\and
\IEEEauthorblockN{Taeho Jung\textsuperscript{*}}
\IEEEauthorblockA{University of Notre Dame\\tjung@nd.edu}
\thanks{\textsuperscript{*}Corresponding author.}}

\maketitle

\begin{abstract}
Reinforcement learning with verifiable rewards (RLVR) trains a language
model on problems that may themselves be confidential, and the trained
model can reveal which problems it saw. We study RLVR under prompt-level
differential privacy: the released weights must be
$(\epsilon,\delta)$-differentially private with respect to the presence
of any one training problem. Taking the group of responses to one prompt
as the privacy record, our method aggregates their gradients, clips the
prompt's contribution once, adds Gaussian noise, and composes the privacy
loss across updates, so the budget depends on neither the number of
responses per prompt nor the clipping norm; to our knowledge this is the
first differential privacy guarantee for RLVR training. We train
Qwen2.5-1.5B-Instruct with LoRA at a per-run budget of $\epsilon=8$ and
compare, on the same prompts and at the same budget, a control that
removes only the reward signal and two private supervised fine-tuning
recipes. The reward signal improves accuracy over the control by 2.65
points on MATH and 3.24 on GSM8K, in every seed; the improvement survives
a format-robust scorer, at 1.3 points on MATH, and is not explained by
response length. At the same budget the private model outperforms both
supervised recipes on MATH and GSM8K by 2.3 to 3.8 points, retains
85--90\% of the gain of non-private GRPO on these tasks, and on MATH the
noise of an eightfold tighter budget costs at most 1.2 points. The reward
effect also carries to CommonsenseQA, an exploratory non-mathematical
task. Verifier feedback thus remains a usable learning signal under
prompt-level privacy.
\end{abstract}

\begin{IEEEkeywords}
differential privacy, reinforcement learning with verifiable rewards,
group relative policy optimization
\end{IEEEkeywords}

\section{Introduction}
\label{sec:intro}

Reinforcement learning with verifiable rewards (RLVR) improves a language
model from its own responses: for each training problem the policy samples
several candidate solutions, a program checks each against the problem's
reference answer, and the policy is updated toward the ones that
pass~\cite{lambert2025tulu,shao2024deepseekmath}. Needing neither
demonstrations nor a learned reward model, the method fits settings where
solution traces are scarce and correctness is cheap to check: assessment
items with answer keys, cases with adjudicated outcomes, tickets with
confirmed resolutions. In such settings the training problems are
themselves confidential, and RLVR reads each one intensively. Every time a
problem is sampled, the policy conditions $G$ generations on it, every
reward is computed against its reference answer, and the resulting update
is written into the released weights. Membership auditing already detects
prompt-specific traces in RLVR-trained
models~\cite{liu2026auditingdatamembershipreinforcement}. A model trained
on confidential problems therefore needs a guarantee about what its
weights reveal about them, and to our knowledge none exists for RLVR.

Two features of RLVR fix where that guarantee must live and for what unit.
In RLHF the sensitive data are human preferences, so private RLHF trains
the reward model under DP~\cite{zhang2025userprivate,cho2026decoupled} and
the policy optimized against it is covered by post-processing. RLVR has
no reward model to privatize: the verifier is public code, and the
sensitive object is the prompt the policy reads at every rollout. The
gradient step is the only stage that touches private data, so it must be
the private mechanism, as in DP-SGD~\cite{abadi2016deep} and private
PPO~\cite{wu2023privately}. The unit is the prompt: a training record is
a problem together with the reference answer its verifier checks, the
data are defined at no finer granularity, and presence of one prompt is
the event the membership attack recovers. Protection at this unit
composes upward; when one person contributes several prompts, a guarantee
for that person follows by group privacy or by accounting at the user
boundary~\cite{charles2025userlevel}. We therefore require
\emph{prompt-level} differential privacy: the released weights must be
$(\epsilon,\delta)$-DP with respect to adding or removing any one training
problem (Definition~\ref{def:prompt-dp}). This is a guarantee about
training items; a guarantee for the people who contribute them is
derived from it, at a cost that grows with the number of items per
person (Section~\ref{sec:threat}).

The obstacle is the group structure of group relative policy optimization
(GRPO)~\cite{shao2024deepseekmath}. DP-SGD protects the unit whose gradient
is clipped~\cite{abadi2016deep}, and the direct extension of private policy
optimization clips each sampled response, as private PPO does with its one
response per prompt~\cite{wu2023privately}. GRPO's $G$ completions per
prompt are not independent contributions: their rewards are normalized
within the group, so changing one reward changes every advantage, and the
loss averages over the group. Removing a prompt removes $G$ clipped terms,
so a per-completion mechanism has sensitivity $GC$ rather than $C$, and
protecting the prompt through it costs group privacy over $G$ records
(Proposition~\ref{prop:group}). User-level DP meets the same problem when
one user contributes several examples and resolves it by clipping the
user's aggregated contribution~\cite{chua2024mind,charles2025userlevel}.
We do the same for the prompt group: the losses of all $G$ completions are
aggregated into one prompt gradient, which is clipped once and perturbed
once. Because GRPO's normalization uses only the group's own rewards, this
gradient depends on the prompt alone given the current parameters, so the
sensitivity is $C$ for every $G$ and the budget depends only on the
sampling rate, the noise multiplier, the number of rounds, and $\delta$.
Group size and clipping norm become utility and compute parameters with no
privacy cost. \ours{} is DP-SGD with this record and the single-update
GRPO loss as the per-record objective, applied to LoRA
parameters~\cite{hu2022lora,yu2022differentially}, and
Theorem~\ref{thm:privacy} gives its guarantee by standard
subsampled-Gaussian accounting.

Whether anything is learned under this mechanism is the open question,
and the answer is not obvious. Each prompt's influence is clipped to a
norm that a single noisy step can dominate, and the learning signal is a
within-group contrast whose sign carries the information. Answering it
requires isolating the reward signal from everything else the private
pipeline does, which existing evaluations of private post-training do not
attempt. We train one 1.5B-parameter instruction model with rank-4 LoRA
under a fixed protocol: a design study on MATH sets the operating point,
and the same arms then run on MATH, on GSM8K as the confirmatory
benchmark, and on CommonsenseQA as an exploratory non-mathematical task,
every arm spending the same per-run budget, $\epsilon=8$, on the same
prompts. A zero-advantage control, the same private pipeline with every
group-relative advantage set to zero after generation and verification,
isolates the reward signal. Two private imitation arms, trained on
verifier-filtered and on reference solutions under the same sampling,
clipping, noise, and accountant, test the other way to spend that budget.
Endpoints are fixed in advance, null results carry minimum detectable
effects, and every gain is checked against a format-robust scorer and
against response length. Three findings result.

(C1)~\emph{Verifier rewards improve the policy under prompt-level
privacy.} \ours{} exceeds the zero-advantage control by $+2.65$ [$1.79$,
$3.51$] pp on MATH and $+3.24$ [$1.61$, $4.89$] pp on GSM8K, with every
seed-level contrast positive, while the control stays within 0.4 points
of the base, as expected once the reward term is removed. The contrast survives a format-robust scorer that
also credits correct unboxed answers, staying resolved at $+1.34$ [$0.54$,
$2.14$] pp on MATH and positive in every seed, and it does not come from
shorter responses: \ours{}'s responses are as long as the base's or
longer. The effect carries beyond mathematics: on CommonsenseQA, \ours{}
exceeds the control by $+4.26$ [$2.50$, $6.06$] pp under a final-answer
rescore.

(C2)~\emph{At a matched budget, reward optimization outperforms private
imitation on the mathematical tasks.} \ours{} exceeds verifier-filtered
and gold DP-SFT by 2.32 to 3.79 pp on MATH and GSM8K, and exceeds MATH
variants trained at its learning rate or for twice as many updates by 2.10
and 1.78 pp. Neither supervised arm moves beyond its detectable range of the base,
even though gold DP-SFT trains on the worked reference solutions as
token-level targets, which \ours{} touches only through the verifier's
answer check. The comparison is between the evaluated recipes at a matched
budget, and the same supervised pipeline does learn where its targets are
short: on CommonsenseQA, whose gold targets are four-token answer
declarations, gold DP-SFT ties \ours{}.

(C3)~\emph{At this operating point, the price of privacy is small on the
mathematical tasks and the budget can be tightened without a measured
loss.} Against non-private GRPO runs measured over their own bases,
\ours{} retains 85\% of the gain on MATH and 90\% on GSM8K. On MATH, the noise of an
eightfold tighter budget, $\epsilon{=}1$, changes accuracy by $-0.57$
[$-1.19$, $0.06$] pp, a cost of at most 1.2 points (a utility measurement
of the procedure at that noise level, Section~\ref{sec:design-stage1}),
and no tested clipping norm changes it beyond its MDE. Where more remains to learn the price rises: on CommonsenseQA
\ours{} retains 45\% of the non-private gain, and a non-private run at
the same hyperparameters is 5.8 points higher.
Section~\ref{sec:discussion} traces this pattern to the ratio of noise to
clipped signal, which shortens each step under AdamW and costs most where
the attainable gain is largest; on the mathematical tasks that gain is
bounded by the 1.5B-parameter base itself.

This paper makes three contributions.
\begin{itemize}
\item \emph{A prompt-level privacy guarantee for RLVR.} Taking the prompt
  group as the record and clipping its aggregated gradient makes the
  released LoRA trajectory $(\epsilon,\delta)$-DP with respect to any one
  training problem under stated assumptions (Section~\ref{sec:threat}),
  with a budget independent of
  the group size and the clipping norm (Algorithm~\ref{alg:dpgrpo},
  Lemma~\ref{lem:round-main}, Theorem~\ref{thm:privacy}). To our knowledge
  this is the first differential privacy guarantee for GRPO or RLVR
  training. The same accounting covers the private supervised baselines
  (Corollary~\ref{cor:baselines}), so every comparison in the paper is at
  a matched guarantee, and mechanism-verification tests check the
  implementation invariants the analysis relies on: cross-prompt
  isolation, the clipping identity, the sensitivity bound and Poisson
  sampling (Appendix~\ref{app:mechanism-verification}).
\item \emph{A protocol that isolates what the reward signal contributes
  under privacy.} The zero-advantage control removes only the advantages
  from an otherwise identical private pipeline; matched budgets, fixed
  endpoints, minimum detectable effects for null results, format-robust
  rescoring and response-length checks separate reward learning from private
  training as such, from answer formatting, and from response length. The
  protocol applies to any private post-training method.
\item \emph{Evidence that reward learning survives prompt-level privacy,
  and where privacy costs.} On two mathematical-reasoning benchmarks and an
  exploratory non-mathematical task, verifier-derived group-relative
  rewards remain a learning signal under the mechanism, outperform the
  evaluated private imitation recipes on the mathematical tasks, retain
  most of the non-private gain there, and are insensitive to the budget
  down to $\epsilon{=}1$ on MATH; the larger cost on CommonsenseQA
  locates where the noise binds.
\end{itemize}

\section{Related Work}
\label{sec:related}

\subsection{Which stage of post-training is private}
Private reward-based post-training differs mainly in which stage reads
private data and therefore carries the guarantee. When the sensitive data
are human preferences, one route privatizes the reward model, at the user
level~\cite{zhang2025userprivate} or through decoupled reward
modeling~\cite{cho2026decoupled}, and the policy trained against it is
covered by post-processing. A second route privatizes the policy update
itself: private PPO on LoRA
parameters~\cite{wu2023privately}, DP-AdamW applied to the DPO and PPO
stages~\cite{chen2025improved}, and trajectory-level DP policy
gradient~\cite{rio2025dppg}, in each case with a learned reward and one
response per record. A third route avoids training on the private data at
all: it fine-tunes a generator under DP, samples synthetic instructions,
and runs ordinary alignment on those~\cite{yu2024privacypreserving}. RLVR
closes off the first and third routes. There is no reward model to
privatize, and a synthetic problem comes without the verified reference
answer that the reward requires; the prompt is what the policy reads at
every rollout (Section~\ref{sec:intro}), so the guarantee must attach to
the policy update, as in the second route, but for a record that GRPO's
group structure defines differently (Section~\ref{sec:record}). The
leakage this addresses is documented: membership auditing detects
prompt-specific traces in RLVR-trained
models~\cite{liu2026auditingdatamembershipreinforcement}, and GRPO on
benign factual data amplifies extraction of memorized personal
data~\cite{zhang2026benign}. Neither work evaluates a defense. Adjacent
lines answer different questions. KL-regularized sampling makes a sampled
action private~\cite{zhang2025klregularization}, not a released weight
trajectory; private RL theory gives sample-complexity and regret
bounds~\cite{he2025policy,he2026differentiallyprivatereinforcementlearning}
without a mechanism for LLM post-training; DP-OPD trains a student with
DP-SGD on its own generations but scores them with a frozen
teacher~\cite{khadem2026dpopd}; and RL used \emph{for} privacy, tuning
DP-SGD~\cite{khadangi2025rldp}, steering a DP text
generator~\cite{hu2025actg} or training
anonymizers~\cite{ghiasvand2026grasp}, does not protect the training
prompts. To our knowledge, no prior work states a differential privacy
guarantee for GRPO or RLVR training.

\subsection{Closest work}
Wu et al.~\cite{wu2023privately} generate responses to private inputs,
train the policy with LoRA and a DP variant of PPO, and release the
policy. We keep that release object and differ in three ways. The learning
signal is GRPO's group-relative advantage, which couples every response to
its prompt and changes both the record and the clipping unit. The reward
is a public verifier, which removes the reward model from the privacy
analysis and from the budget. And the evaluation includes a matched
intervention that removes only the advantages, so the contribution of the
reward signal inside the private procedure is measured rather than
assumed.

\subsection{The privacy record}
DP-SGD protects whatever is clipped~\cite{abadi2016deep}. User-level DP
enlarges the record to all of one user's examples and shows that the
enlargement changes both the guarantee and the best training
recipe~\cite{chua2024mind,charles2025userlevel}. GRPO poses the question
from the other side: the responses to one prompt are generated from it
and normalized against each other, and the GRPO gradient is a U-statistic
over the group~\cite{zhou2026ustatistic}, so per-response clipping does
not bound the prompt's influence. Our record is the prompt with all of its
responses, and it differs from a user-level record in kind, not only in
name: a user's examples are fixed data, whereas a prompt's responses are
drawn afresh from the current policy in every round, so the record has no
fixed members, no examples-per-unit parameter arises, and the guarantee
must hold over the randomness of generation, which is what assumptions
(A2) and (A3) of Section~\ref{sec:privacy-guarantee} secure. With
the record fixed, low-rank adaptation~\cite{yu2022differentially}, scale
and optimizer choices~\cite{li2022large,de2022unlocking}, and
privacy-loss-distribution
accounting~\cite{sommer2019privacy,koskela2020computing,doroshenko2022connect}
carry over unchanged.

\subsection{What a credible gain requires}
On the privacy side, membership
attacks~\cite{liu2026auditingdatamembershipreinforcement} demonstrate
leakage and one-run auditing~\cite{steinke2023privacy} lower-bounds it;
neither upper-bounds it, so neither replaces the mechanism proof, and we
report our canary diagnostic on those terms
(Appendix~\ref{app:mechanism-verification}). On the utility side, an
accuracy gain after RLVR is easy to misattribute. The large gains reported
for small Qwen models start from base rather than instruction-tuned
checkpoints~\cite{zeng2025simplerlzoo}; part of any gain is format
correction~\cite{wang2025oneshot}; RLVR mainly reweights solutions the
base already samples~\cite{yue2025doesrl}; Qwen2.5-Math improves even
under random rewards~\cite{shao2026spurious}, which has been attributed to
contamination~\cite{wu2026reasoning}; and small gains move with seed,
decoding, answer extraction~\cite{hochlehnert2025sober} and GRPO's length
bias~\cite{liu2025understanding}. These findings set the protocol of
Section~\ref{sec:experiments}: an already post-trained starting
model~\cite{qwen2024qwen25}, a control that holds generation,
verification, clipping and noise fixed and removes only the advantages,
fixed endpoints across seeds, and format-robust and length checks.

\subsection{Private supervised alternatives}
Supervised fine-tuning is the other way to spend a prompt-level budget:
raise the likelihood of a target response instead of weighting sampled
responses by their rewards, under DP-SGD with parameter-efficient
adaptation~\cite{abadi2016deep,li2022large,yu2022differentially}. Our two
comparators (Section~\ref{sec:experiments}) differ in where the targets
come from. Verifier-filtered DP-SFT descends from rejection-sampling
fine-tuning~\cite{yuan2023scaling} and the self-training loops of
STaR~\cite{zelikman2022star} and ReST$^{EM}$~\cite{singh2024beyond},
applied for one round under the same prompt-level mechanism; gold DP-SFT
imitates the dataset's reference solutions. DeepSeekMath places
supervised, rejection-sampling and policy-gradient updates in one
framework~\cite{shao2024deepseekmath}, and their empirical ranking is
setting-dependent: self-training can beat imitation of human
solutions~\cite{singh2024beyond}, and RL-versus-SFT comparisons turn on
the targets and the evaluation
distribution~\cite{havrilla2024teaching,chu2025sft,hochlehnert2025sober}.
Our comparison holds the starting model, prompt set, clipping mechanism
and per-run budget fixed and, because the supervised clipping norm and
learning rates differ from GRPO's, compares specified private procedures
rather than objectives in the abstract.

\section{Background and Preliminaries}
\label{sec:background}

This section fixes the objects the rest of the paper uses: the
differential privacy definition and the two facts about it we need
(Section~\ref{sec:dp}), the GRPO objective (Section~\ref{sec:grpo}), the
DP-SGD mechanism and its privacy profile (Section~\ref{sec:dpsgd}), and,
from these, the privacy record for GRPO and the guarantee that
Theorem~\ref{thm:privacy} proves (Section~\ref{sec:record}).

\subsection{Differential Privacy}
\label{sec:dp}

A dataset $D$ is a finite multiset of records from a universe $\mathcal X$.
Two datasets $D$ and $D'$ are \emph{neighbors}, written $D\simeq D'$, if one
is obtained from the other by adding or removing exactly one record;
Definition~\ref{def:prompt-dp} specializes $\mathcal X$ and this relation
to prompts.

\begin{definition}[Differential privacy~\cite{dwork2014algorithmic}]
\label{def:dp}
A randomized mechanism $\mathcal M$ is $(\epsilon,\delta)$-differentially
private if, for all $D\simeq D'$ and every measurable set $O$ of outputs,
\begin{equation}
\Pr[\mathcal M(D)\in O]\le e^{\epsilon}\Pr[\mathcal M(D')\in O]+\delta.
\label{eq:dp}
\end{equation}
\end{definition}

\begin{proposition}[Group privacy]
\label{prop:group}
If $\mathcal M$ is $(\epsilon,\delta)$-DP and $D,D'$ differ in $k$ records,
then for every measurable $O$,
$\Pr[\mathcal M(D)\in O]\le e^{k\epsilon}\Pr[\mathcal M(D')\in O]
+ke^{(k-1)\epsilon}\delta$.
\end{proposition}

\noindent Chaining Definition~\ref{def:dp} over $k$ add-or-remove steps gives
the additive term $\delta\sum_{j<k}e^{j\epsilon}$; compare
\cite[Lemma~2.2]{vadhan2017complexity}.

\begin{proposition}[Post-processing~\cite{dwork2014algorithmic}]
\label{prop:post}
If $\mathcal M$ is $(\epsilon,\delta)$-DP and $f$ is any randomized map
that does not depend on $D$, then $f\circ\mathcal M$ is
$(\epsilon,\delta)$-DP.
\end{proposition}

\subsection{RLVR and GRPO}
\label{sec:grpo}

A policy $\pi_\theta$ is an autoregressive language model: given a prompt
$x$ it defines a distribution over completions $y=(y_1,\dots,y_L)$ through
$\pi_\theta(y\mid x)=\prod_{k=1}^{L}\pi_\theta(y_k\mid x,y_{<k})$. In
reinforcement learning with verifiable rewards
(RLVR)~\cite{lambert2025tulu}, each prompt carries a reference answer, and a
programmatic verifier $r(x,y)\in\{0,1\}$ reports whether the final answer in
$y$ matches it. Training maximizes the expected reward
$\mathbb E_{x}\,\mathbb E_{y\sim\pi_\theta(\cdot\mid x)}[r(x,y)]$ under a
penalty of weight $\beta$ on the divergence of $\pi_\theta$ from a fixed
reference policy $\pi_{\rm ref}$, here the public base model.

Group relative policy optimization (GRPO)~\cite{shao2024deepseekmath}
estimates this objective without a value function. For prompt $x_i$ it
draws $G$ completions $y_{i,1},\dots,y_{i,G}$ independently from a sampling
policy $\pi_{\rm old}$, computes rewards $r_{i,g}=r(x_i,y_{i,g})$, and
forms group-normalized advantages
\begin{equation}
A_{i,g}=\frac{r_{i,g}-\mu_i}{s_i+c},
\label{eq:adv}
\end{equation}
where $\mu_i$ and $s_i$ are the mean and standard deviation of
$\{r_{i,g}\}_{g=1}^G$ and $c>0$ is a stabilizing constant. Every token of
$y_{i,g}$ receives the advantage $A_{i,g}$. With the per-token probability
ratio $\rho_{i,g,k}=\pi_\theta(y_{i,g,k}\mid x_i,y_{i,g,<k})/
\pi_{\rm old}(y_{i,g,k}\mid x_i,y_{i,g,<k})$ and completion length
$L_{i,g}$, the objective for prompt $x_i$
is~\cite{shao2024deepseekmath,schulman2017proximal}
\begin{multline}
J_i(\theta)=\frac1G\sum_{g=1}^G\frac{1}{L_{i,g}}\sum_{k=1}^{L_{i,g}}
\Big[\min\!\big(\rho_{i,g,k}A_{i,g},\\
\operatorname{clip}(\rho_{i,g,k},1-\kappa,1+\kappa)A_{i,g}\big)
-\beta K_{i,g,k}\Big],
\label{eq:grpo}
\end{multline}
where $K_{i,g,k}=\pi_{\rm ref}/\pi_\theta-\log(\pi_{\rm ref}/\pi_\theta)-1$,
evaluated at token $y_{i,g,k}$, is the nonnegative per-token KL estimator
of Schulman~\cite{schulman2020kl} adopted by
DeepSeekMath~\cite{shao2024deepseekmath}. Under samples from $\pi_\theta$ it
is an unbiased estimate of the token-level
$\mathrm{KL}(\pi_\theta\Vert\pi_{\rm ref})$; its gradient is not an unbiased
estimate of the gradient of that divergence~\cite{tang2025pitfallskldivergencegradient}. We use
it as implemented, since the privacy analysis of
Section~\ref{sec:privacy-guarantee} holds for any per-prompt loss. When one gradient
step is taken per batch of rollouts, as in this paper, $\pi_{\rm old}$ and
$\pi_\theta$ coincide at the point of differentiation, so $\rho_{i,g,k}=1$,
the ratio clip is inactive, and
$\nabla_\theta\rho_{i,g,k}=\nabla_\theta\log\pi_\theta(y_{i,g,k}\mid
x_i,y_{i,g,<k})$.

\subsection{DP-SGD and the Subsampled Gaussian Mechanism}
\label{sec:dpsgd}

DP-SGD~\cite{abadi2016deep} trains parameters $\theta$ on a dataset of $n$
records with per-record losses $\ell_i(\theta)$. Round $t$ performs four
steps. (i)~\emph{Poisson sampling:} each record $i$ is included in $S_t$
independently with probability $q$. (ii)~\emph{Per-record clipping:} for
$i\in S_t$, the gradient $g_i=\nabla_\theta\ell_i(\theta_t)$ is projected to
$\ell_2$ norm at most $C$,
\begin{equation}
\bar g_i=g_i\min\!\left(1,\frac{C}{\lVert g_i\rVert_2}\right).
\label{eq:clip}
\end{equation}
(iii)~\emph{Noisy aggregation:} the round releases
$\widehat g_t=\bigl(\sum_{i\in S_t}\bar g_i+z_t\bigr)/(qn)$ with
$z_t\sim\mathcal N(0,\sigma^2C^2I)$, where $\sigma$ is the noise
multiplier and $qn$ is the public expected sample size.
(iv)~\emph{Update:} $\theta_{t+1}$ is computed from $\theta_t$ and
$\widehat g_t$ by a public optimizer.

Adding or removing one record changes the sum inside (iii) by a vector of
norm at most $C$. Rotating that vector onto a coordinate axis and dividing by
$C$ reduces one round, for fixed $\theta_t$, to a pair of one-dimensional
output laws, and $T$ adaptively composed rounds to their $T$-fold
products~\cite{balle2018privacy,zhu2022optimal}, whose hockey-stick
divergence is the privacy profile~\cite{balle2018privacy}.

\begin{definition}[Privacy profile of DP-SGD]
\label{def:profile}
For probability measures $P,Q$ and $\alpha\ge0$, let
$H_\alpha(P\Vert Q)=\sup_O\,[P(O)-\alpha\,Q(O)]$. For $q\in(0,1]$ and
$\sigma>0$, let
\begin{equation}
A_{q,\sigma}=(1-q)\,\mathcal N(0,\sigma^2)+q\,\mathcal N(1,\sigma^2),
\qquad
B_\sigma=\mathcal N(0,\sigma^2).
\label{eq:AB}
\end{equation}
The $T$-round privacy profile is, for $\epsilon\ge0$,
\begin{equation}
\delta_{q,\sigma,T}(\epsilon)=\max\bigl\{
H_{e^\epsilon}(A_{q,\sigma}^{\otimes T}\Vert B_\sigma^{\otimes T}),\,
H_{e^\epsilon}(B_\sigma^{\otimes T}\Vert A_{q,\sigma}^{\otimes T})\bigr\}.
\label{eq:profile}
\end{equation}
\end{definition}

Since $H_{e^\epsilon}(\mathcal M(D)\Vert\mathcal M(D'))$ is the smallest
$\delta$ for which~\eqref{eq:dp} holds on the ordered pair
$(D,D')$~\cite{balle2018privacy}, $T$ rounds of DP-SGD are
$(\epsilon,\delta_{q,\sigma,T}(\epsilon))$-DP for every $\epsilon\ge0$, and
the bound is attained in the worst case over losses~\cite{zhu2022optimal}.
The two terms in~\eqref{eq:profile} are the remove and add directions. The
profile depends on $(q,\sigma,T)$ and not on $C$, because sensitivity and
noise both scale with $C$. Accountants evaluate it through R\'enyi
divergences~\cite{mironov2017renyi,mironov2019sampled} or through the
privacy loss distribution~\cite{sommer2019privacy,koskela2020computing};
Section~\ref{sec:accounting} states which one this paper reports.

Two remarks on the mechanism are used below. First, the sensitivity bound
requires that, given $\theta_t$, the clipped vector $\bar g_i$ be a
function of record $i$ alone; a statistic shared across records before
clipping, such as a batch-level normalization, would let one record change
another's contribution and void the bound. Second, the clipping norm $C$ in
\eqref{eq:clip} is unrelated to the ratio clip $\kappa$ in \eqref{eq:grpo};
only $C$ enters the privacy analysis.

\paragraph{Low-rank adaptation}
LoRA~\cite{hu2022lora} freezes a pretrained weight matrix
$W\in\mathbb R^{d\times k}$ and trains an update $W+BA$ with
$B\in\mathbb R^{d\times r}$, $A\in\mathbb R^{r\times k}$, and
$r\ll\min(d,k)$. Under DP-SGD only the adapter parameters are clipped and
perturbed, and the frozen base weights are public. With $p$ trainable
coordinates the noise $z_t$ has expected squared norm $\sigma^2C^2p$, so
restricting $p$ lowers the noise at a fixed budget, which is the basis of
parameter-efficient private
fine-tuning~\cite{yu2021lowrank,yu2022differentially}.

\subsection{The Privacy Record in GRPO}
\label{sec:record}

DP-SGD protects the unit whose gradient is clipped in \eqref{eq:clip}, and
a unit spanning $k$ such records is protected only through
Proposition~\ref{prop:group}. The record is therefore a design choice.
User-level DP enlarges it from one example to the aggregated contribution
of all examples of one user, which is clipped as a
whole~\cite{chua2024mind,charles2025userlevel}; there the members of a
record are fixed data. Private PPO takes the prompt as the record but
generates only one response per prompt~\cite{wu2023privately}. GRPO's
record has neither property: its members are $G$ responses sampled from
the current policy, so they change every round and are coupled to one
another through the group statistics.

In GRPO the natural per-example unit, one completion $y_{i,g}$, is not an
independent contribution. The $G$ completions of prompt $x_i$ share the
input from which they are sampled and against which they are verified,
share the statistics $\mu_i$ and $s_i$ of \eqref{eq:adv}, so that changing
one reward changes every advantage of the group, and share the average in
\eqref{eq:grpo}. If each completion were a record clipped to $C$, removing
$x_i$ would remove $G$ clipped terms from the sum in step (iii) of
Section~\ref{sec:dpsgd}, giving sensitivity $GC$; a guarantee for the
prompt would then follow only from Proposition~\ref{prop:group} with
$k=G$, or from noise scaled to $GC$.

We therefore take the record to be the prompt with all of its
completions: the gradient $g_i$ of the per-prompt loss $-J_i(\theta_t)$,
which already aggregates the $G$ completions, is the quantity clipped in
\eqref{eq:clip}. Its sensitivity is
$C$ for every $G$, so the group size does not enter the privacy budget.
The first remark of Section~\ref{sec:dpsgd} holds because $\mu_i$ and $s_i$
are computed within the group of $x_i$; batch-level reward normalization
or advantage whitening across prompts would break it. A record is
accordingly a problem together with the reference answer its verifier
checks, and we write $x_i\in\mathcal X$ for the pair. Because Poisson
sampling runs over a fixed index set, we state the guarantee over a public
universe of $n$ slots, each holding at most one record.

\begin{definition}[Prompt-level differential privacy]
\label{def:prompt-dp}
Fix a public $n$. A \emph{prompt dataset} is $D=((b_i,x_i))_{i=1}^n$ with
presence bits $b_i\in\{0,1\}$ and records $x_i\in\mathcal X$, where $x_i$
is read only if $b_i=1$. Prompt datasets $D$ and $D'$ are neighbors,
$D\simeq D'$, if $b_k\ne b'_k$ for exactly one slot $k$ and
$(b_i,x_i)=(b'_i,x'_i)$ for every $i\ne k$; the record at slot $k$ on
the side where $b_k=0$ is never read and may be arbitrary. A randomized
training algorithm $\mathcal A$ with a public round count $T$ that maps a
prompt dataset to a weight trajectory $(\theta_2,\dots,\theta_{T+1})$ is
\emph{prompt-level $(\epsilon,\delta)$-DP} if~\eqref{eq:dp} holds with
$\mathcal M=\mathcal A$ for every pair of neighboring prompt datasets and
every measurable set $O$ of trajectories.
\end{definition}

Definition~\ref{def:prompt-dp} is Definition~\ref{def:dp} with the record
fixed to a prompt group and the output fixed to the whole released
trajectory. Section~\ref{sec:threat} states what the adversary sees,
Section~\ref{sec:method} gives the algorithm, and
Theorem~\ref{thm:privacy} proves that it is prompt-level
$(\epsilon,\delta_{q,\sigma,T}(\epsilon))$-DP.

\section{Threat Model}
\label{sec:threat}

\paragraph{Setting.}
We adopt the central model with a trusted curator, the convention of
DP-SGD~\cite{abadi2016deep,charles2025userlevel}. The curator holds a
prompt dataset $D$ over $n$ public slots (Definition~\ref{def:prompt-dp}),
with $n$ chosen without access to the data, and executes
Algorithm~\ref{alg:dpgrpo}; an absent slot contributes the zero vector to
every sum. The release is the full trajectory of LoRA weights
$\theta_2,\dots,\theta_{T+1}$ together with any checkpoint derived from
it. Neighbors keep the capacity $n$ and differ in one presence bit: the
standard DP-SGD relation with the record enlarged from one example to one
prompt group, under which Poisson sampling over $[n]$ reveals nothing
about the realized dataset size.

\paragraph{Protected event.}
The protected event is the presence of one prompt in $D$, the value of a
single bit $b_k$. Everything that prompt induces is protected with it: the
$G$ rollouts sampled from it in every round in which it is drawn, their
rewards and advantages, and its clipped gradient. The output is the whole
trajectory because a present prompt is resampled independently in every
round, so protecting its presence in one batch would not bound what the
trajectory reveals. Membership is the right event for RLVR because it is
what an adversary already recovers from an unprotected model: the policy
reads each prompt at every rollout, and membership auditing detects
prompt-specific traces in RLVR-trained
models~\cite{liu2026auditingdatamembershipreinforcement}. Concretely,
prompt-level $(\epsilon,\delta)$-DP bounds every test of $b_k$ that the
adversary below can run on the release: its true-positive rate is at most
$e^{\epsilon}$ times its false-positive rate plus
$\delta$~\cite{dong2022gaussian,balle2020hypothesis}. We report $\epsilon{=}8$, the operating point at which private
fine-tuning of language models is commonly
reported~\cite{li2022large,yu2022differentially}; the budget sweep of
Section~\ref{sec:design-stage1} measures what the noise of a tighter
budget costs.

\paragraph{Adversary.}
The adversary is white-box, adaptive, and computationally unbounded. It
knows every public input of Algorithm~\ref{alg:dpgrpo}: the base model,
the LoRA initialization $\theta_1$, the reference policy, the verifier,
the optimizer, the hyperparameters $(q,C,\sigma,T,G,\tau,\beta)$, the slot
count $n$, and the accountant. It knows every prompt in $D$ except the one
whose presence is in question, may inspect and run the weights, issue
unlimited queries, and combine the results with arbitrary side
information; each such use is post-processing
(Proposition~\ref{prop:post}) and leaves the guarantee unchanged. It does
not observe the run's internal randomness: the sampling coins of line~2,
the rollout randomness of line~4, and the Gaussian noise of line~8.
Recovery files that store these generator states are not part of the
release, and publishing them would void the guarantee.

\paragraph{Limits of the guarantee.}
Five limits remain, and each is a risk the proof does not remove.
(i)~\emph{The unit is a prompt, not a person.} If one person contributes
$k$ prompts, that person's guarantee follows only through group privacy
(Proposition~\ref{prop:group}) and degrades with $k$; protecting people
requires grouping and accounting at the user
boundary~\cite{charles2025userlevel}. (ii)~\emph{The theorem certifies the
analyzed mechanism under assumptions (A1)--(A4) of
Section~\ref{sec:privacy-guarantee}.} As in the private fine-tuning
literature~\cite{li2022large,yu2022differentially}, the budget covers the
reported run and not the pilot runs on which its hyperparameters were
selected; charging those would need a private tuning
analysis~\cite{papernot2022hyperparameter}. Two execution effects lie
outside the analysis: batched inference kernels and the training
verifier's wall-clock limit can couple one prompt's rollouts and rewards
to other records, and noise is generated in floating point rather than
over the reals~\cite{mironov2012significance,jin2022are}. Neither is
quantified; Appendix~\ref{app:assumptions} records the status of each
assumption.
(iii)~\emph{The bound is certified, not measured.} Every $\epsilon$ is
a privacy-loss-distribution budget with pessimistic discretization,
reproduced by an independent accountant and by R\'enyi accounting
(Section~\ref{sec:accounting}, Appendix~\ref{app:rdp}), so the guarantee
is as strong as the proof and the accounting make it. How much slack that
bound leaves in practice, which an attack such
as~\cite{liu2026auditingdatamembershipreinforcement} run against the
private model would measure, is left to future work
(Section~\ref{sec:limitations}). (iv)~\emph{The curator is trusted.} A
compromised curator, or a released copy of the recovery files, is outside
the model. (v)~\emph{Not every private-looking quantity is a protected
output.} The verifier is public code; rewards, completions, completion
lengths, per-record gradients and system telemetry are training-time
quantities that are never released and are not covered as outputs; and
the bound concerns what the weights reveal about membership, not what the
base model already knew about a prompt. The corpora of
Section~\ref{sec:experiments} are public benchmarks used as a mechanism
testbed; deployment on data whose subjects can be harmed requires a
privacy unit aligned with those subjects.

\section{The \ours{} Algorithm}
\label{sec:method}

\ours{} is DP-SGD with the prompt group as the clipped record and the
single-update GRPO loss as the per-record objective. Algorithm~\ref{alg:dpgrpo}
defines it, and Figure~\ref{fig:mechanism} shows one round.

\begin{algorithm}[t]
\caption{\ours{}: GRPO under prompt-level DP}
\label{alg:dpgrpo}
\begin{algorithmic}[1]
\REQUIRE Private prompt dataset $D=((b_i,x_i))_{i=1}^n$
(Definition~\ref{def:prompt-dp}). Public: LoRA initialization $\theta_1$, reference
policy $\pi_{\rm ref}$, verifier $r$, group size $G\ge2$, sampling
temperature $\tau$, advantage constant $c$, sampling rate $q$, clipping
norm $C$, noise multiplier $\sigma$, rounds $T$, KL weight $\beta$,
optimizer $\mathrm{Opt}$.
\ENSURE Weights $\theta_2,\dots,\theta_{T+1}$.
\STATE \textbf{for} $t=1,\dots,T$ \textbf{do}
\STATE \hspace{\algorithmicindent} $S_t\leftarrow\{i\in[n]:c_{t,i}=1\}$ with
  independent $c_{t,i}\sim\mathrm{Bernoulli}(q)$
\STATE \hspace{\algorithmicindent} \textbf{for each} $i\in S_t$ with $b_i=1$
  \textbf{do}
\STATE \hspace{2\algorithmicindent} sample
  $y_{i,1},\dots,y_{i,G}\sim\pi_{\theta_t}^{\tau}(\cdot\mid x_i)$ independently
\STATE \hspace{2\algorithmicindent} $r_{i,g}\leftarrow r(x_i,y_{i,g})$;\;
  $A_{i,g}\leftarrow(r_{i,g}-\mu_i)/(s_i+c)$ as in~\eqref{eq:adv}
\STATE \hspace{2\algorithmicindent} $g_i\leftarrow\nabla_\theta\ell_i(\theta_t)$
  with $\ell_i$ from~\eqref{eq:loss}
\STATE \hspace{2\algorithmicindent}
  $\bar g_i\leftarrow g_i\min\bigl(1,C/\lVert g_i\rVert_2\bigr)$
\STATE \hspace{\algorithmicindent} $z_t\sim\mathcal N(0,\sigma^2C^2I)$;\;
  $\widehat g_t\leftarrow\bigl(\sum_{i\in S_t}b_i\bar g_i+z_t\bigr)/(qn)$
\STATE \hspace{\algorithmicindent}
  $\theta_{t+1}\leftarrow\mathrm{Opt}(\theta_t,\widehat g_t)$
\end{algorithmic}
\end{algorithm}

\begin{figure*}[t]
\centering
\includegraphics[width=\textwidth]{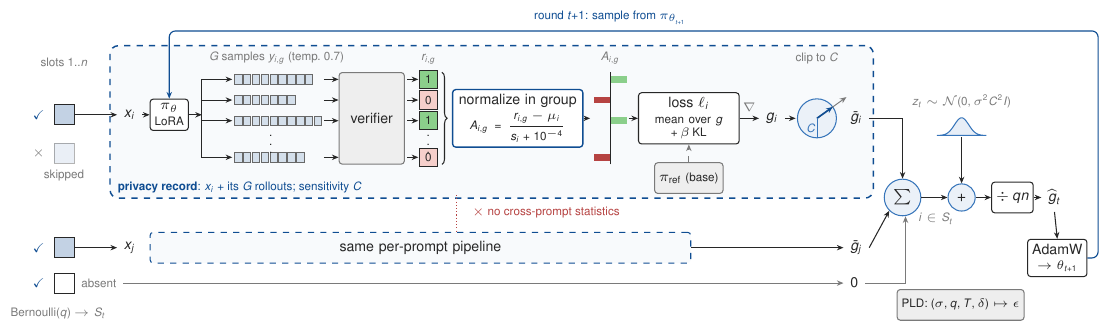}
\caption{One round of \ours{} (Algorithm~\ref{alg:dpgrpo}). Each
Poisson-sampled prompt $x_i$ (line~2) generates $G$ completions whose verifier
rewards are normalized within that group only (lines~4--5). The rollout losses
are combined into one prompt gradient $g_i$ and clipped to norm $C$
(lines~6--7), so the dashed region is a single privacy record. Clipped prompt
gradients are summed over the sample, perturbed with Gaussian noise, divided by
the public constant $qn$, and applied to the LoRA parameters (lines~8--9). The
next round samples rollouts from the updated policy.}
\label{fig:mechanism}
\end{figure*}

Line~6 uses the implemented single-update loss. With
$u_{i,g,k}=\log\pi_\theta(y_{i,g,k}\mid x_i,y_{i,g,<k})$, completion length
$L_{i,g}$, and the per-token KL estimate $K_{i,g,k}$ of~\eqref{eq:grpo},
\begin{equation}
\ell_i(\theta)=-\frac{1}{G}\sum_{g=1}^G\frac{1}{L_{i,g}}\sum_{k=1}^{L_{i,g}}
\left[e^{u_{i,g,k}-\operatorname{sg}(u_{i,g,k})}A_{i,g}
-\beta K_{i,g,k}\right],
\label{eq:loss}
\end{equation}
where $\operatorname{sg}$ stops gradients; advantages and reference
log-probabilities are detached. This is~\eqref{eq:grpo} with one update per
rollout batch, so the ratio equals one and the ratio clip is inactive;
Appendix~\ref{app:reproducibility} records the implementation. Line~4 samples
from $\pi_{\theta_t}$ at temperature $\tau=0.7$; $\pi^{\tau}_{\theta_t}$ is a
public function of $\theta_t$. In line~5, $s_i$ uses denominator $G-1$ and
$c=10^{-4}$; constant-reward groups are retained with zero advantages.
Rollouts are backpropagated in chunks that accumulate into $g_i$ before
line~7. Only LoRA coordinates are trainable. Line~8 runs even when $S_t$ is
empty, so $qn$ is a public constant, and division by it, like line~9, is
post-processing of the noisy sum.

The zero-advantage control sets $A_{i,g}=0$ in~\eqref{eq:loss}, leaving the KL term
and lines~7--9 intact. Its prompt contribution has the same norm bound, so
both arms use the same accountant parameters. This intervention tests whether
verifier-derived advantages contribute to utility after privatization.

\section{Privacy Analysis}
\label{sec:privacy-guarantee}

We now prove that Algorithm~\ref{alg:dpgrpo} satisfies
Definition~\ref{def:prompt-dp}, the guarantee the threat model of
Section~\ref{sec:threat} requires.

\subsection{Assumptions}

The guarantee holds under four assumptions about how
Algorithm~\ref{alg:dpgrpo} is run. The end of
Appendix~\ref{app:privacy-proof} records how each is enforced, tested, or
left as a disclosed gap.
\begin{itemize}
\item[(A1)] \emph{Public design.} Every public input of
Algorithm~\ref{alg:dpgrpo}, including all hyperparameters, is chosen without
access to $D$.
\item[(A2)] \emph{Record isolation.} Given $\theta_t$, the clipped gradients
$\bar g_i$ of line~7 are independent across prompts, and the law of
$\bar g_i$ depends only on $(\theta_t,x_i)$.
\item[(A3)] \emph{Fresh hidden randomness.} The coins of line~2, the rollout
randomness of line~4, and the noise of line~8 are drawn afresh in every
round, mutually independent, independent of $D$, and never released.
\item[(A4)] \emph{Exact arithmetic.} Lines~4--9 are computed over the reals,
and $z_t$ is an exact Gaussian draw.
\end{itemize}

\subsection{Guarantee}

Write $A=A_{q,\sigma}$ and $B=B_\sigma$ for the pair of
Definition~\ref{def:profile}.

\begin{lemma}[One round of \ours{}]
\label{lem:round-main}
Assume (A1)--(A4). Fix a round $t$, any values of the prior releases
$\widehat g_{1:t-1}$, and prompt datasets $D\simeq D'$ such that slot $k$
is present in $D$ and absent in $D'$. Let $M$ and $M'$ be the conditional
laws of $\widehat g_t$ under $D$ and $D'$. Then, for every $\alpha\ge0$,
\begin{equation}
H_\alpha(M\Vert M')\le H_\alpha(A\Vert B),
\qquad
H_\alpha(M'\Vert M)\le H_\alpha(B\Vert A).
\label{eq:round}
\end{equation}
\end{lemma}

Conditioned on everything released so far, one round thus reveals no more
about the prompt in slot $k$ than one DP-SGD step does about a worst-case
record, whatever $t$, the past releases, $G$, or $C$; in the terminology
of Zhu et al.~\cite{zhu2022optimal}, $(A,B)$ dominates the round.

\begin{theorem}[\ours{} is prompt-level DP]
\label{thm:privacy}
Assume (A1)--(A4). For every $\epsilon\ge0$, Algorithm~\ref{alg:dpgrpo}
is prompt-level $(\epsilon,\delta_{q,\sigma,T}(\epsilon))$-DP
(Definition~\ref{def:prompt-dp}).
\end{theorem}

\begin{IEEEproof}
Definition~\ref{def:prompt-dp} requires~\eqref{eq:dp} on every ordered
pair of neighbors. Let $D\simeq D'$ with slot $k$ present in $D$ and
absent in $D'$; the reversed pair is handled by the same argument with
the roles of $A$ and $B$ exchanged. By (A1), $\theta_t$ is a fixed
public function of $\widehat g_{1:t-1}$ (line~9 iterated from $\theta_1$),
so round $t$, conditioned on the prior releases, is the mechanism of
Lemma~\ref{lem:round-main}, and~\eqref{eq:round} holds for every value of
$\widehat g_{1:t-1}$. The adaptive composition theorem for dominating
pairs~\cite{zhu2022optimal} (Lemma~\ref{lem:comp} in
Appendix~\ref{app:privacy-proof}) combines these per-round bounds into
\begin{equation*}
H_{e^\epsilon}\bigl(\mathcal L_D(\widehat g_{1:T})\,\big\Vert\,
\mathcal L_{D'}(\widehat g_{1:T})\bigr)
\le H_{e^\epsilon}(A^{\otimes T}\Vert B^{\otimes T})
\le \delta_{q,\sigma,T}(\epsilon),
\end{equation*}
and likewise with $D,D'$ and $A,B$ exchanged. By the characterization
of~\eqref{eq:dp} through $H_{e^\epsilon}$ (Section~\ref{sec:dpsgd}), the
releases $\widehat g_{1:T}$ therefore satisfy~\eqref{eq:dp} with
$\delta=\delta_{q,\sigma,T}(\epsilon)$ on every neighboring pair. The
trajectory $\theta_{2:T+1}$, the output named in
Definition~\ref{def:prompt-dp}, is computed from $\widehat g_{1:T}$ by the
public maps of line~9, which do not read $D$, so
Proposition~\ref{prop:post} carries the bound to it; with $T$ public by
(A1), this is exactly the statement of Definition~\ref{def:prompt-dp}
with $\delta=\delta_{q,\sigma,T}(\epsilon)$.
\end{IEEEproof}

Lemma~\ref{lem:round-main} is proved in Appendix~\ref{app:privacy-proof},
together with the two corollaries cited below. Any checkpoint selection
that reads only the released weights and public data is further
post-processing and inherits the guarantee. Three consequences of
Theorem~\ref{thm:privacy} organize the rest of the paper. First,
$\delta_{q,\sigma,T}$ involves neither the group size $G$, the clipping
norm $C$, the loss, nor the verifier, so $G$ is a compute and utility
parameter rather than a privacy cost. Second, because the proof reads
lines~4--6 only through (A2)--(A4), any training procedure that keeps
lines~2--3 and 7--9 and replaces lines~4--6 by a per-record computation
inherits the same guarantee at the same $(q,\sigma,T)$; the
verifier-filtered DP-SFT and gold-SFT comparators of
Section~\ref{sec:results} are such procedures
(Corollary~\ref{cor:baselines}), so equal $\epsilon$ across arms means an
identical mechanism, with the caveat that gold-SFT's record carries the
full reference solution rather than a final answer. Third,
Definition~\ref{def:prompt-dp} uses add/remove neighbors; replacing one
prompt costs $(2\epsilon,(1+e^{\epsilon})\delta)$ by group privacy
(Corollary~\ref{cor:replace}).

\subsection{Accounting}
\label{sec:accounting}
Every privacy budget in this paper is the profile $\delta_{q,\sigma,T}$ of
Definition~\ref{def:profile}, evaluated by privacy-loss-distribution (PLD)
accounting~\cite{sommer2019privacy,koskela2020computing} with the
implementation in Google's differential-privacy library~\cite{googledp}
and the add-or-remove neighboring relation. The accountant discretizes the
privacy loss by the connect-the-dots
construction~\cite{doroshenko2022connect} with pessimistic rounding that
assigns truncated tail mass to infinite loss, so every reported $\epsilon$
upper-bounds the budget of the ideal mechanism. Given a slot universe of
size $n$, we set $\delta=1/(10n)$ and
$q=m/n$ for a chosen expected sample size $m$, fix the round count $T$, and
take $\sigma$ to be the smallest noise multiplier, found by bisection, at
which the accountant returns the target $\epsilon$. Two tasks with the same
$(n,m,T)$ and target therefore receive the same $\sigma$ and the same
certified budget, whatever their prompts, verifier, or loss. For every
configuration used in this paper the accountant returns its target
$\epsilon$ to at least four decimal places; Table~\ref{tab:rdp} in Appendix~\ref{app:privacy-proof} lists
each configuration with its certified budget and, for readers who reproduce
our numbers with an RDP accountant, the corresponding R\'enyi-DP value. The
implementation reproduces the published JAX-Privacy~\cite{balle2025jaxprivacy} reference value to
within $3.1\times10^{-4}$ (Appendix~\ref{app:mechanism-verification}).

\section{Experimental Design}
\label{sec:experiments}

The experiments ask whether verifier feedback improves a policy trained
under prompt-level differential privacy, whether private imitation at the
same budget matches it, and how much of that comparison depends on answer
scoring and on training choices. Within a benchmark, every compared arm
spends the same per-run privacy budget on the same prompts. A \emph{design
study} (Sec.~\ref{sec:design-stage1}) on MATH fixed the operating point and
the supervised recipe; a \emph{fixed protocol}
(Sec.~\ref{sec:design-stage2}) then ran five arms on MATH and GSM8K, and on
CommonsenseQA as an exploratory extension. MATH is therefore the
development benchmark, whose fixed-protocol results replicate the design
study, and GSM8K is the confirmatory one. Table~\ref{tab:arms}
(Appendix~\ref{app:historical-results}) lists every arm, its seeds, and
where it is reported, and Table~\ref{tab:protocol-timeline}
(Appendix~\ref{app:protocol-timeline}) lists every protocol decision and
the runs it preceded or followed.

\subsection{Common Setup}
\label{sec:design-setup}

\paragraph{Benchmarks.}
A benchmark supplies $n$ training prompts, a deterministic verifier
$r(x,y)\in\{0,1\}$ that checks a response's final answer against the
prompt's reference answer, and a disjoint evaluation set scored by the
same verifier. The prompt with its reference answer is the private record
(Sec.~\ref{sec:record}); the verifier is public.
MATH~\cite{hendrycks2021math} and GSM8K~\cite{cobbe2021gsm8k} are the main
benchmarks. CommonsenseQA~\cite{talmor2019commonsenseqa}, five-way multiple
choice, tests the comparison outside mathematics: its 6,961 training
questions are sampled so that $n$, and with it every mechanism parameter,
equals GSM8K's, and it is evaluated on the 1,221-question validation split
because test labels are not public. Its verifier reads one option letter
from an ``Answer: $\langle$letter$\rangle$'' declaration, a format adopted
after a pilot with a stricter answer marker failed.
SVAMP~\cite{patel2021svamp} is used only to evaluate GSM8K checkpoints.

\paragraph{Model and training.}
Every arm fine-tunes rank-4 LoRA adapters on Qwen2.5-1.5B-Instruct
(configuration in Appendix~\ref{app:config-ledger}). GRPO draws $G{=}8$
completions per prompt at temperature $0.7$ with KL weight $\beta{=}0.04$.
The learning rate is $3\times10^{-4}$ for \ours{} and the zero-advantage
control, chosen by a pilot on MATH and carried unchanged to the other
benchmarks; the supervised arms use the learning rates and clipping norm
of their own recipes, selected as described in
Appendix~\ref{app:config-pilots}. Each run is one $(\epsilon,\delta)$-DP
release of the LoRA weights, with $\sigma$ set by bisection so that $T$
rounds compose to the target $\epsilon$ (Sec.~\ref{sec:accounting});
Table~\ref{tab:rdp} certifies the budget of every configuration. The
budget is per run; pilot tuning and any released randomness are outside
it (Sec.~\ref{sec:threat}).

\paragraph{Shared and per-benchmark settings.}
Every benchmark uses $\epsilon{=}8$ at $\delta{=}1/(10n)$, $qn{=}128$
expected prompts per round, $G{=}8$, and $C{=}0.055$ for \ours{} and its
control (the supervised arms' norm is in
Appendix~\ref{app:config-params}). Within a benchmark all arms share $T$,
$q$, $\sigma$ and the training prompts, and may differ in $C$ without
changing $\epsilon$ because the noise scales with $C$
(Corollary~\ref{cor:baselines}); for the same reason, arms with
different $C$ also share the ratio of noise to the largest clipped
signal, and $C$ sets only how much of each record's gradient survives
clipping. Table~\ref{tab:config} lists what
differs: the horizon ($T{=}50$ on MATH, $T{=}100$ on the other two), the
$\sigma$ it implies, and a completion cap matched to solution length;
these are the parameters every result in Section~\ref{sec:results}
refers to. The horizons are short by design: at a fixed $\epsilon$ more
rounds require a larger $\sigma$, and 50 to 100 rounds keep $\sigma$
near 0.5 while spending the budget on under two expected passes over
the prompts ($qT$ in Table~\ref{tab:config}).
GSM8K's protocol, including its horizon, was fixed before any private
GSM8K run. We set $\delta{=}1/(10n)$, an order of magnitude below $1/n$.

\begin{table}[t]
\centering
\caption{Per-benchmark settings of the private runs. Shared by all
three: $\epsilon{=}8$ at $\delta{=}1/(10n)$, $qn{=}128$, $G{=}8$,
$C{=}0.055$ (\ours{}, control), learning rate $3\times10^{-4}$.
Full-precision values, the supervised arms' recipe parameters and split
construction are in Appendix~\ref{app:config-ledger}; certified budgets
in Table~\ref{tab:rdp}.}
\label{tab:config}
\footnotesize
\setlength{\tabcolsep}{3pt}
\begin{tabular}{@{}lrrr@{}}
\toprule
 & MATH & GSM8K & CSQA \\
\midrule
Role                     & development & confirmatory & exploratory \\
Training prompts $n$     & 6{,}988 & 6{,}961 & 6{,}961 \\
Evaluation set           & 4{,}499 test & 1{,}319 test & 1{,}221 val. \\
Rounds $T$               & 50 & 100 & 100 \\
Expected passes $qT$     & 0.92 & 1.84 & 1.84 \\
Noise multiplier $\sigma$ & 0.484 & 0.520 & 0.520 \\
Completion cap (tokens)  & 1{,}024 & 512 & 512 \\
Strict scorer            & boxed answer & final number & option letter \\
\bottomrule
\end{tabular}
\end{table}

\subsection{Stage 1: Design Study}
\label{sec:design-stage1}

MATH was the development benchmark, and every design run used two
initialization seeds per setting. The \emph{budget sweep} sets
$\epsilon\in\{1,4.23,8\}$ at fixed $T$ by changing $\sigma$ alone; these
runs come from the earlier design study and measure the utility of the
procedure at each noise level (Appendix~\ref{app:config-shared}). The
\emph{clipping sweep} sets
$C\in\{\tfrac12,1,2\}\times0.055$ at fixed $\epsilon$, which moves clipping
bias and absolute noise together. The \emph{allocation study} replaces
eight rollouts on 128 expected prompts per round with two rollouts on 512
(the same rollout count) or on 128 (a quarter of it); each change also
changes $q$ and $\sigma$. Two \emph{supervised challenges} run filtered
DP-SFT at \ours{}'s learning rate (same-LR) and at $T{=}100$ with $\sigma$
recalibrated to keep $\epsilon{=}8$ (100-update). Finally,
\emph{non-private references} remove the noise ($\sigma{=}0$) at the private
hyperparameters, with clipping kept on MATH, kept or disabled on
GSM8K, and disabled on CommonsenseQA; they make no privacy claim, and the
GSM8K references, as the one exception to final-step evaluation, are
scored at steps 25, 50 and 100. A two-seed clipped MATH reference and a
single-seed unclipped GSM8K reference, both at learning rate $10^{-4}$,
test a smaller step. Appendix~\ref{sec:res-budget} reports the
budget, clipping and non-private results and Section~\ref{sec:res-c2} the
supervised challenges.

\subsection{Stage 2: Fixed Protocol}
\label{sec:design-stage2}
\label{sec:design-arms}

Five arms run on each benchmark at that benchmark's settings.
The \emph{public base} is the untuned model. \emph{\ours{}} is
Algorithm~\ref{alg:dpgrpo}. The \emph{zero-advantage control} runs the
identical private pipeline (generation, verification, KL term, clipping,
noise and optimizer updates) but sets every advantage in
Eq.~\eqref{eq:adv} to zero before the loss, so correct responses receive
no preference over incorrect ones. It is a trained control, not the public
base; it shares \ours{}'s procedure and hyperparameters but not its
realized noise or rollouts, and the difference between the two measures
the reward-advantage term. The LoRA update starts at zero (PEFT's default
initialization of $B$), so the control starts at the reference policy,
where the KL gradient vanishes: its first update is pure noise, and later
updates add only the KL pull toward the base. The control therefore answers the question the private setting raises
first: whether the private pipeline itself, its noise, clipping and KL
pull, moves the model without a reward signal, so that a gain of \ours{}
over the base can be attributed to the reward-advantage term rather than
to private training as such. A control close to the base is the expected
outcome, and the contrast with \ours{} is then the reward effect as a
whole; which part of the advantage carries it is the finer question that
the advantage ladder of Section~\ref{sec:limitations} would take up.

Two supervised arms train the same LoRA with prompt-level DP-SGD at the
same $q$, $\sigma$ and $T$, hence the same $\epsilon$, following
\cite{li2022large,yu2022differentially}, without a KL term.
\emph{Verifier-filtered DP-SFT} adapts rejection-sampling fine-tuning
\cite{yuan2023scaling}: it samples $G{=}8$ public-base completions per
prompt, keeps those the verifier accepts, and minimizes their
negative log-likelihood averaged within the prompt before clipping; a
prompt with no accepted completion contributes a zero gradient.
\emph{Gold DP-SFT} trains on the reference target, which \ours{} never
sees: a worked solution on MATH and GSM8K, and the four-token declaration
``Answer: $\langle$gold letter$\rangle$'' on CommonsenseQA.

\paragraph{Seeds.}
Every arm has two initialization seeds per benchmark. On MATH and GSM8K,
\ours{} and the control have two more, added after the two-seed results
had been inspected, so their four-seed contrasts are sequential,
exploratory extensions. Supervised contrasts pair the two SFT seeds with
\ours{} seeds 1--2, which share their initializations. Per-seed values
are in Appendix~\ref{app:per-seed}.

\subsection{Endpoints and Scoring}
\label{sec:design-endpoints}

The \emph{primary endpoint} is strict accuracy of one greedy completion
per evaluation problem at the final training step, with no checkpoint or
problem selection and a scorer frozen before the arms were run. Strict
scoring credits a correct extracted answer only: a boxed answer checked by
symbolic equality on MATH, a final number on GSM8K, an option letter on
CommonsenseQA. A response with no extraction counts as incorrect, and every
problem stays in the denominator. MATH is scored on 4,499 test problems,
excluding those that earlier runs used for checkpoint selection
(Appendix~\ref{app:config-ledger}); GSM8K on its full test split.

Sensitivity analyses rescore the same saved responses and change no
training reward. On MATH, the \emph{format-robust} endpoint also credits a
completed unboxed final answer under the same symbolic check; responses
cut at the cap stay ineligible, and the rules were audited before
arm-level rescoring (Appendix~\ref{app:format-robust}). On CommonsenseQA,
a \emph{final-answer} rescore reads the last explicit answer declaration
and abstains on ambiguous ones. It was chosen after the strict scores of
four arms had been inspected, then hash-frozen before their rescoring and
before filtered-SFT evaluation, so it is a post-hoc sensitivity endpoint
(Appendix~\ref{app:csqa-extraction}). GSM8K's scorer already credits an
unboxed final number, so it has no separate format-robust endpoint. For
transfer, the final GSM8K
checkpoints are scored on SVAMP's 1,000 problems with GSM8K's prompt and
cap (Appendix~\ref{app:svamp}). Every compared arm, the base included, is
scored in the same evaluation session, and accuracies are never compared
across benchmarks.

\subsection{Statistics}
\label{sec:design-stats}

A contrast is the paired difference in accuracy between two arms on the
shared evaluation set, in percentage points (pp). Brackets give a 95\%
problem-cluster bootstrap interval over the evaluation set (100{,}000
draws; each arm's seeds averaged within a problem); the four-seed
contrasts additionally report the standard deviation of the per-seed
differences and a seed-level $t$ interval. Each
contrast also carries exact McNemar tests on each seed pair,
Holm-adjusted within its family, and a sign-randomization $p$-value for
the pooled difference. A family is the set of seed-pair tests of one
contrast on MATH and GSM8K, and the eight tests of the final-answer
rescore on CommonsenseQA, where the rescore was fixed after the arms had
been seen. Beside every contrast whose interval includes zero we give the
two-sided 80\%-power minimum detectable effect (MDE), so such a result
reads as a bound on the effect rather than as its absence.
\FloatBarrier 

\section{Results}
\label{sec:results}

MATH and GSM8K provide the main evidence; CommonsenseQA tests a
non-mathematical setting, and SVAMP tests transfer without further
training. Fig.~\ref{fig:results} places every arm against the public base,
and Table~\ref{tab:contrasts} gives \ours{}'s gain over each comparator
(positive values favor \ours{}). Strict scoring is primary throughout;
arms, scorers, intervals and MDEs are as defined in
Section~\ref{sec:experiments}, every private arm runs at the parameters
of Table~\ref{tab:config}, and its certified budget is in
Table~\ref{tab:rdp}.

\begin{figure*}[t]
\centering
\includegraphics[width=\textwidth]{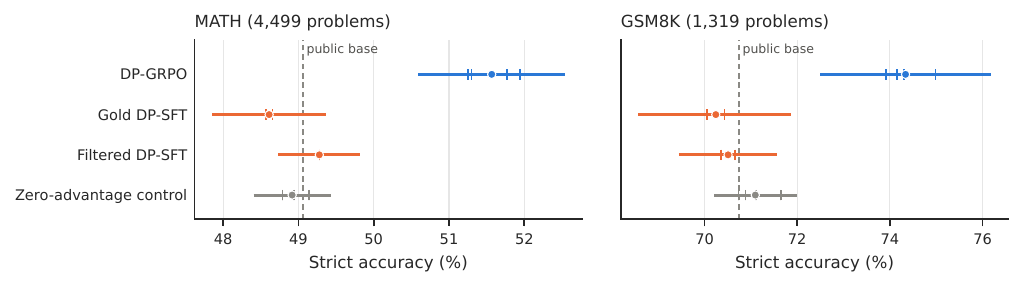}
\caption{Strict accuracy of each arm at $\epsilon{=}8$. Dot: mean over
seeds; ticks: individual seeds; dashed line: public base. Each bar is the
95\% paired interval of the arm's difference from the base, added to the
base value, so a bar that clears the dashed line is a resolved gain and
one that crosses it is not. Exact values in Appendix~\ref{app:per-seed}.}
\label{fig:results}
\end{figure*}

\subsection{Reward Learning Survives Prompt-Level Privacy}
\label{sec:res-c1}
\begin{table}[t]
\centering
\caption{\textbf{\ours{}} accuracy gain over each comparator.}
\label{tab:contrasts}
\footnotesize
\setlength{\tabcolsep}{4pt}
\begin{tabular}{@{}lll@{}}
\toprule
Comparator & MATH (pp) & GSM8K (pp) \\
\midrule
Zero-adv.\ control & $+2.65$ [$1.79$, $3.51$] & $+3.24$ [$1.61$, $4.89$] \\
Filtered DP-SFT & $+2.32$ [$1.39$, $3.26$] & $+3.53$ [$1.74$, $5.31$] \\
Gold DP-SFT     & $+2.99$ [$2.01$, $3.98$] & $+3.79$ [$1.97$, $5.57$] \\
\bottomrule
\end{tabular}
\par\smallskip
\raggedright\footnotesize
Values are strict-accuracy differences in percentage points (pp), with
95\% paired CIs. Positive values favor \ours{}. The control row pools
four \ours{} and four control seeds, two of each added sequentially
(Sec.~\ref{sec:design-arms}); the supervised rows pair two seeds each.
\end{table}

At the per-run $\epsilon{=}8$ operating point, \ours{} improves strict
accuracy over the public base by $+2.51$ [$1.54$, $3.48$] pp on MATH and
$+3.60$ [$1.76$, $5.44$] on GSM8K, reaching four-seed means of 51.6\% and
74.3\% from bases of 49.1\% and 70.7\%. The zero-advantage control, which
trains through the same private pipeline without any reward preference,
stays at the base, $-0.14$ [$-0.65$, $0.37$] and $+0.36$ [$-0.53$,
$1.27$] with MDEs 0.73 and 1.29, as expected for a pipeline whose only
remaining signal is the KL pull toward the base
(Section~\ref{sec:design-arms}): noise and clipping alone neither help
nor harm at this budget. \ours{} exceeds it by
$+2.65$ [$1.79$, $3.51$] on MATH and $+3.24$ [$1.61$, $4.89$] on GSM8K, so
within this pipeline the gain is attributable to the reward-advantage
term rather than to private training as such.

The effect is consistent across training runs. Every seed-level reward
contrast is positive on both benchmarks, and every seed-level McNemar
test rejects equality after Holm adjustment, the largest adjusted $p$
being 0.013 on GSM8K; the sign-randomization $p$ for the pooled contrast
is below $10^{-3}$ on both. The per-seed differences have a standard
deviation of 0.36 pp on MATH and 0.72 on GSM8K, and seed-level $t$
intervals of [2.08, 3.22] and [2.10, 4.38], so the reward effect is
seven and four times its seed-to-seed spread. These pools include the two
sequentially added seeds (Sec.~\ref{sec:design-arms}); per-seed contrasts
and tests are in Appendix~\ref{app:per-seed}.

\subsection{Comparison with Private Supervised Fine-Tuning}
\label{sec:res-c2}

\ours{} outperforms both private imitation procedures on MATH and GSM8K
(Table~\ref{tab:contrasts}). All eight seed-paired supervised contrasts,
two seed pairs for each of the four task-by-target comparisons, favor it
and are significant after Holm adjustment, with pooled gains from 2.32 to
3.79 pp. The supervised arms themselves stay within their MDEs of
the base. Filtered DP-SFT changes MATH by $+0.22$ [$-0.32$, $0.77$] and
GSM8K by $-0.23$ [$-1.29$, $0.83$] (MDEs 0.78 and 1.50); gold DP-SFT,
despite training on the worked reference solutions as token-level
targets, which \ours{} touches only through the verifier's answer check,
changes
them by $-0.44$ [$-1.20$, $0.31$] and $-0.49$ [$-2.16$, $1.14$] (MDEs
1.08 and 2.36). The same supervised pipeline does learn when its targets are short: on
CommonsenseQA, gold DP-SFT gains $+4.95$ pp over the base and matches
\ours{} (Section~\ref{sec:res-csqa}). The flat mathematical results
therefore reflect what these recipes extract from long worked solutions
under this noise, not a pipeline that fails to train;
Section~\ref{sec:discussion} discusses why. These results compare specified procedures at a matched per-run
budget and guarantee (recipe selection in
Appendix~\ref{app:config-pilots}); they concern these recipes rather than
private imitation in general.

Two MATH challenges test whether the gap closes at \ours{}'s learning
rate or with twice as many updates at recalibrated noise. Neither does:
\ours{} remains $+2.10$ [$1.14$, $3.07$] above the same-LR arm and
$+1.78$ [$0.88$, $2.67$] above the 100-update arm. The longer run is the
only supervised change with a resolved improvement, $+0.54$ [$0.03$,
$1.06$] over the selected recipe (Table~\ref{tab:design}).

\subsection{Answer Formatting and Correctness}
\label{sec:res-c3a}
\begin{table}[t]
\centering
\caption{MATH accuracy under two answer extractors.}
\label{tab:robust}
\footnotesize
\setlength{\tabcolsep}{4pt}
\begin{tabular}{@{}lrr@{}}
\toprule
Model & Strict (\%) & Format-robust (\%) \\
\midrule
Public base     & $49.06$ & $50.37$ \\
Zero-adv.\ control & $48.92\;(-0.14)$ & $50.28\;(-0.08)$ \\
Filtered DP-SFT & $49.28\;(+0.22)$ & $50.47\;(+0.10)$ \\
Gold DP-SFT     & $48.61\;(-0.44)$ & $49.98\;(-0.39)$ \\
\ours{}         & $51.57\;(+2.51)$ & $51.62\;(+1.26)$ \\
\midrule
Non-private GRPO$^{\dagger}$ & $52.12\;(+2.96)$ & -- \\
\bottomrule
\end{tabular}
\par\smallskip
\raggedright\footnotesize
Parentheses give the difference from the public base in percentage points
on 4,499 problems. Both extractors score the same saved responses.
$^{\dagger}$Earlier run of standard GRPO, one seed: LoRA rank 16,
LR $5\times10^{-6}$, checkpoint 900, with its gain measured against its
own base (49.17\%; Appendix~\ref{app:config-shared}); later checkpoints score 51.4--52.7\%. It was not
format-robustly rescored.
\end{table}

\begin{table}[t]
\centering
\caption{GSM8K accuracy.}
\label{tab:gsm8k}
\footnotesize
\setlength{\tabcolsep}{4pt}
\begin{tabular}{@{}lr@{}}
\toprule
Model & Accuracy (\%) \\
\midrule
Public base & $70.74$ \\
Zero-adv.\ control & $71.10\;(+0.36)$ \\
Filtered DP-SFT & $70.51\;(-0.23)$ \\
Gold DP-SFT & $70.24\;(-0.49)$ \\
\ours{} & $74.34\;(+3.60)$ \\
\midrule
Non-private GRPO$^{\dagger}$ & $75.36\;(+4.02)$ \\
\bottomrule
\end{tabular}
\par\smallskip
\raggedright\footnotesize
Parentheses give the difference from the public base in percentage points
on 1,319 questions. The private arms use 100 updates;
\ours{} and the zero-advantage control average four seeds, the
supervised arms two.
$^{\dagger}$Earlier run of standard GRPO, one seed: LoRA rank 16,
LR $5\times10^{-6}$, two epochs (1{,}740 updates). Its gain is measured
against its own base (71.34\%; Appendix~\ref{app:config-shared}).
\end{table}

Two robustness checks separate answer format and response length from
correctness. Tables~\ref{tab:robust} and~\ref{tab:gsm8k} give the MATH and
GSM8K accuracies; their non-private rows are discussed in
Section~\ref{sec:res-nonprivate}.

The reward effect survives the format-robust scorer. The reward contrast
is $+1.34$ [$0.54$, $2.14$] pp, positive in all four seeds, and the gain
over the base $+1.26$ [$0.33$, $2.17$]. Both are smaller than under strict
scoring because the rescore lifts the base from 49.06\% to 50.37\% and the
control from 48.92\% to 50.28\%, while \ours{}, whose correct answers are
nearly all boxed already, moves only to 51.62\%.
Table~\ref{tab:math-format-decomposition} (Appendix~\ref{app:format-robust})
locates the gain. On the 3,831 problems the base already answers in the
required form, \ours{} changes accuracy by $-0.24$ pp; on the 668
problems the base leaves without a boxed answer it adds $+1.49$ pp under
format-robust scoring and $+2.75$ under strict. The reward signal thus
teaches the model to produce a checkable answer on problems where the
base produces none, while leaving the problems the base already answers
essentially unchanged. \ours{} also stays resolved above
selected-LR filtered and gold DP-SFT, by $+1.19$ [$0.30$, $2.08$] and
$+1.68$ [$0.74$, $2.61$]; against the 100-update challenge the gap is
$+0.78$ [$-0.07$, $1.62$], MDE 1.22, so that comparison depends on the
scorer.

GSM8K needs no separate rescoring: its scorer already credits an
unboxed final number, so its strict endpoint is format-robust by
construction, and its reward contrast, $+3.24$ [$1.61$, $4.89$] pp in
every seed, is the format-independent evidence for C1. Nor does the MATH
gain come from shorter responses: \ours{}'s responses average 554 tokens
against the base's 540.

\subsection{Outside Mathematics: CommonsenseQA}
\label{sec:res-csqa}
\begin{table}[t]
\centering
\caption{CommonsenseQA accuracy under two answer extractors.}
\label{tab:csqa}
\footnotesize
\setlength{\tabcolsep}{4pt}
\begin{tabular}{@{}lrr@{}}
\toprule
Model & Strict (\%) & Final answer (\%) \\
\midrule
Public base      & $46.19$ & $66.75$ \\
Zero-adv.\ control & $45.99\;(-0.20)$ & $67.20\;(+0.45)$ \\
Filtered DP-SFT  & $52.42\;(+6.22)$ & $66.71\;(-0.04)$ \\
Gold DP-SFT      & $71.70\;(+25.51)$ & $71.70\;(+4.95)$ \\
\ours{}        & $70.84\;(+24.65)$ & $71.46\;(+4.71)$ \\
\midrule
Non-private GRPO & $77.27\;(+31.08)$ & $77.27\;(+10.52)$ \\
\bottomrule
\end{tabular}
\par\smallskip
\raggedright\footnotesize
Parentheses give the difference from the public base in percentage points.
Both extractors score the same saved responses on 1,221 questions.
The non-private reference keeps \ours{}'s hyperparameters and seeds and
sets $\sigma=0$ without clipping.
\end{table}

On CommonsenseQA the strict gain over base, $+24.7$ [$22.1$, $27.3$] pp,
comes mostly from extraction: the strict extractor recognizes only 64\%
of base responses, partly because discussion of several options reads as
conflicting answers. Under the final-answer rescore
(Table~\ref{tab:csqa}), a resolved reward effect remains. \ours{} reaches
71.5\% against 66.7\% for base and 67.2\% for the control, with gaps of
$+4.71$ [$2.87$, $6.59$] over base, $+4.26$ [$2.50$, $6.06$] over the
control, and $+4.75$ [$2.99$, $6.55$] over filtered DP-SFT, each with
Holm-adjusted $p<0.001$ within the eight-test family. Even crediting all
21 remaining base abstentions as correct leaves a descriptive 2.99-point
gap below \ours{}; this bound does not cover errors in accepted answers
(Appendix~\ref{app:csqa-extraction}). Gold DP-SFT, whose targets are
four-token answer declarations, reaches 71.7\%, and the \ours{}$-$gold gap
is $-0.25$ [$-1.56$, $1.02$]. The reward effect thus carries beyond
mathematics. That gold-label SFT matches it here is what the target
structure predicts: a four-token declaration supervises exactly the
scored answer, so imitation and reward optimization have the same thing
to learn from each record, whereas a mathematical reference solution
spends most of its tokens on the derivation, so under the same clipped
and noised gradient the answer receives a smaller share of the signal
(Section~\ref{sec:discussion}).

\subsection{Comparison with Non-Private Training}
\label{sec:res-nonprivate}

Each benchmark table ends with one non-private GRPO run. On MATH and
GSM8K it is standard GRPO without privacy, scored against its own base
(configuration in Appendix~\ref{app:config-shared}); on CommonsenseQA it
keeps \ours{}'s hyperparameters and seeds and removes the noise and the
clipping. Measured against each run's own base, \ours{} retains most of
the non-private gain on the mathematical tasks: $+2.51$ against $+2.96$
pp on MATH (85\%) and $+3.60$ against $+4.02$ on GSM8K (90\%). On
CommonsenseQA it retains 45\%, $+4.71$ against $+10.52$ under the
final-answer rescore. Only this last comparison is paired: the
non-private run exceeds \ours{} by $+6.43$ [$4.67$, $8.23$] pp strict and
$+5.81$ [$4.10$, $7.58$] under the rescore, a resolved privacy cost,
though measured at the private hyperparameters rather than against a
tuned ceiling. The mathematical-task ratios are descriptive and carry no interval;
only the CommonsenseQA comparison is paired.

\subsection{Supporting Studies}
\label{sec:res-supporting}

On MATH, accuracy is insensitive to the budget over the tested range:
moving from $\epsilon{=}8$ to $\epsilon{=}1$ changes it by $-0.57$
[$-1.19$, $0.06$] pp, so an eightfold tighter guarantee costs at most 1.2
points, and halving or doubling the clipping norm changes it by less than
its MDE (Appendix~\ref{sec:res-budget}). The $\epsilon{=}1$ run doubles the
noise multiplier and still reaches the same accuracy, which fits the
account in Section~\ref{sec:discussion}: the attainable gain on MATH is
small, and \ours{} reaches most of it under either noise level.
Drawing two rollouts per prompt instead of eight cuts measured training
time 3.2$\times$ at a strict-accuracy cost of 0.96 [0.18, 1.74] pp, about
half of which is answer formatting, since the cost falls to 0.49
[$-0.27$, $1.26$] under format-robust scoring
(Appendix~\ref{sec:res-alloc}). Transfer of the GSM8K checkpoints to
SVAMP without further training gains $+1.20$ [$-0.53$, $2.95$] pp over
base, positive in all four seeds; with 1,000 problems the MDE is 2.49 pp,
so a transfer gain of GSM8K's size would have been detected and one of a
third of that size, which the point estimate shows, would not
(Appendix~\ref{app:svamp}).

\section{Discussion and Limitations}
\label{sec:limitations}
\label{sec:discussion}

\paragraph{When imitation is effective.}
The supervised comparisons fit one account. Filtered imitation can only
reinforce prompts the public base already answers correctly, whereas
reward optimization exploits contrasts among responses from the evolving
policy, including prompts the base answers inconsistently. Gold imitation
supplies a target for every prompt, but what that target supervises
decides how much of the clipped, noised gradient reaches the scored
answer: CommonsenseQA's four-token declarations supervise the answer and
nothing else, so gold DP-SFT matches \ours{} there, while a mathematical
reference solution spends most of its tokens on the derivation, so the
answer receives a small share of a gradient whose norm is capped at $C$.
A target-length study at fixed answer supervision would measure this
share directly.

\paragraph{When privacy costs accuracy.}
The cost of privacy appears to track how much a task leaves to learn
(Section~\ref{sec:res-nonprivate}). On MATH and GSM8K the non-private
runs gain only three to four points over the base, and under its recipe
the MATH run stays between 51.4\% and 52.7\% from 900 to 3{,}494
updates, so training longer without privacy does not raise that target.
\ours{} reaches most of this small gain within 50 or 100 noisy updates.
On CommonsenseQA the non-private run gains 10.5 points under the
final-answer rescore, and \ours{} captures less than half of it in 100
updates. A plausible mechanism is the noise. On MATH and GSM8K its norm is
eight to nine times the largest possible clipped signal
(Appendix~\ref{app:config-shared}), and under AdamW it shortens each step
along the signal. Shorter steps cost little when the attainable gain is
small, as on the mathematical tasks, and more when there is more to
learn, as on CommonsenseQA, whose answers of about 100 tokens never reach
the cap. Our experiments do not
isolate this explanation.

The small non-private gains also mark a limit of our setting. Every run
uses one 1.5B-parameter model, and its strength likely bounds what reward
learning of any kind, private or not, can add on these benchmarks. A
stronger base model may have more to gain, and whether privacy then costs
more, as on CommonsenseQA, or stays small is untested. Within this bound
the method is useful: under a prompt-level guarantee, \ours{} keeps most
of the non-private gain on both mathematical tasks, varies little across
seeds, and improves on the matched control on all three tasks.

\paragraph{Scope and inference.}
All results come from one model, Qwen2.5-1.5B-Instruct with rank-4 LoRA,
whose strength may bound the attainable gains (see above).
The evidence for C1 and C2 comes from two mathematical-reasoning
benchmarks whose verifiers check a final answer; CommonsenseQA serves
only to check that the reward effect carries beyond mathematics
(Sec.~\ref{sec:res-csqa}). The design study ran on MATH, with only non-private references on
GSM8K and CommonsenseQA, so the privacy-budget, clipping and allocation results describe one
benchmark. Transfer was checked once, by scoring the GSM8K checkpoints
without further training on SVAMP~\cite{patel2021svamp}
(Appendix~\ref{app:svamp}): \ours{} improves on the base there by a
third of its GSM8K gain, within the interval over problems, and the
advantage over DP-SFT is not detectable, so C2 is an in-distribution
claim. The
budget and clipping-norm nulls exclude large effects over the tested
ranges, not small ones (Appendix~\ref{sec:res-budget}).

\paragraph{Measurement.}
The verifier defines correctness, and strict scoring also rewards answer
format. Under the format-robust endpoint the MATH \ours{}$-$base gain
falls from 2.51 to 1.26 points (Table~\ref{tab:robust}); GSM8K's scorer
already credits unboxed answers, so it has no such endpoint. The MATH
evaluation set excludes the 500 problems that earlier runs used for
checkpoint selection (Appendix~\ref{app:config-ledger}), and the earlier
design-study runs in Table~\ref{tab:design} are compared only with each
other (Appendix~\ref{app:config-shared}).

\paragraph{Comparators.}
The supervised arms use recipe parameters selected on MATH and carried
unchanged to the other benchmarks (Appendix~\ref{app:config-pilots}), so
C2 is a statement about these recipes at a matched guarantee. The
non-private references keep the private hyperparameters rather than being
tuned for noiseless training, so the CommonsenseQA cost is measured
against a reference that a tuned run could raise, and the
mathematical-task ratios are descriptive (Section~\ref{sec:res-nonprivate});
a non-private learning-rate sweep would turn both into interval
estimates. The horizon differs between benchmarks ($T{=}50$ and $100$).

\paragraph{Future work.}
Each of these gaps has a direct experiment. A one-run audit with enough
canaries to have power at this horizon~\cite{steinke2023privacy} would
put an empirical lower bound beside the proof, and a membership
attack~\cite{liu2026auditingdatamembershipreinforcement} run on both the
non-private and the private model would show how much measurable leakage
the mechanism removes. The zero-advantage control extends naturally into an advantage ladder:
adding a shuffled-reward arm, which keeps the advantage distribution but
detaches it from correctness, and a positive-only unnormalized advantage
that scores each verified completion 1, between the control and the full
group-normalized advantage, would show how much of the reward effect
comes from verifier correctness and how much from the within-group
contrast itself, refining the attribution the control already provides. A compute-matched DP-SFT
with its pilot budget charged to
$\epsilon$~\cite{papernot2022hyperparameter} would make the supervised
comparison end to end. Verifiers beyond final-answer
mathematics, such as unit tests for code, would test whether the result
depends on short binary rewards. Finally, moving the record from one prompt
to one contributor's prompts by clipping at the user
boundary~\cite{charles2025userlevel} is the step that a corpus whose
subjects can be harmed would require (Sec.~\ref{sec:threat}).

\section{Conclusion}
\label{sec:conclusion}

Prompt-level privacy for RLVR depends on choosing the record: taking the
group of responses to one prompt as the record makes GRPO an instance of
DP-SGD whose budget depends on neither the group size nor the clipping
norm, and yields, to our knowledge, the first differential privacy
guarantee for RLVR training. Under that guarantee, at a per-run budget of
$\epsilon=8$ on one 1.5B-parameter model, the verifier's reward signal
still improves accuracy on MATH and GSM8K over a matched control in every
seed, the improvement survives format-robust rescoring and length
controls, and at the same budget it outperforms the evaluated private
imitation recipes on both tasks. \ours{} retains 85--90\% of the
non-private gain on the mathematical tasks, where the noise of an eightfold
tighter budget changes MATH accuracy by at most 1.2 points, and less than half on
CommonsenseQA, where more remains to learn. The protocol that separates these effects, a zero-advantage
control at a matched budget with format-robust and response-length checks,
applies to any private post-training method. Larger base models, which
leave more to learn, are the natural next test of both the method and
its cost.

\section*{LLM usage considerations}
LLMs were used for editorial purposes in this manuscript, and all outputs
were inspected by the authors to ensure accuracy and originality. We also
used an LLM (Anthropic's Claude) beyond editing: it drafted and revised
prose, helped write and review the training, evaluation and analysis code,
and produced the tables and figures from the saved outputs. The authors
set the research questions, the mechanism, the experimental design and
every claim in the paper; reviewed the code, whose privacy-critical paths
are covered by the mechanism-verification tests of
Appendix~\ref{app:mechanism-verification}; approved every job before it
ran; and checked all reported numbers and references against their
sources. The model we train, Qwen2.5-1.5B-Instruct, is the object of study
and contributes no text or judgment to this paper. Every run trains on a
single GPU, the longest for about 12 hours
(Appendix~\ref{app:config-shared}). The authors take full responsibility
for the paper.

\section*{Ethical Considerations}
All training and evaluation data are public benchmarks, and no human
subjects were involved. We treat the training prompts as private only to
evaluate the mechanism. The guarantee protects one prompt, not one person:
a person who contributes $k$ prompts is protected only through group
privacy, which degrades with $k$ (Proposition~\ref{prop:group}). A
deployment on confidential problems whose authors or subjects could be
harmed therefore needs a privacy unit aligned with them, such as
user-level accounting~\cite{charles2025userlevel}. The mechanism releases
weights only; the training telemetry it reads must stay unreleased.

\bibliographystyle{IEEEtran}
\bibliography{refs}

\appendices
\section{Proofs for Section~\ref{sec:privacy-guarantee}}
\label{app:privacy-proof}

\subsection{Notation}
\label{app:notation}

Throughout, $D\simeq D'$ are neighboring prompt datasets in the sense of
Definition~\ref{def:prompt-dp}, and $d$ is the number of LoRA
coordinates. Because that relation is symmetric and~\eqref{eq:dp} is
required on every ordered pair, it suffices to treat pairs in which
slot $k$ is present in $D$ and absent in $D'$ and to bound both
$H_\alpha(\cdot\Vert\cdot)$ orders; this is what Lemma~\ref{lem:round-main}
does. We add a round index to the quantities of
Algorithm~\ref{alg:dpgrpo}: $c_{t,i}$ is the coin of line~2, $\bar g_{t,i}$
the clipped gradient of line~7, and $z_t$ the noise of line~8. Round $t$
computes the unnormalized release
\begin{equation}
U_t=\sum_{i=1}^n c_{t,i}\,b_i\,\bar g_{t,i}+z_t,
\qquad
\widehat g_t=\frac{U_t}{qn}.
\label{eq:app-round}
\end{equation}
Since $qn$ is public, $U_t\mapsto\widehat g_t$ is a bijection with
measurable inverse, so by Lemma~\ref{lem:hs}(i) below every hockey-stick
divergence between round laws is the same whether computed on $U_t$ or on
$\widehat g_t$; we compute on $U_t$, where the algebra is simplest, and
state every result for the release $\widehat g_t$, which is the object of
Section~\ref{sec:privacy-guarantee}. Write $h_{<t}=\widehat g_{1:t-1}$
for the prior releases, $\mathcal L_D(X)$ for the law of a random
variable $X$ when the input is $D$, and
$\mathcal M_t(D\mid h_{<t})=\mathcal L_D(\widehat g_t\mid h_{<t})$ for the
conditional law of round $t$. By (A1), $\theta_t=\Phi_t(h_{<t})$ for a
fixed public map $\Phi_t$ that iterates line~9 from $\theta_1$; if public
evaluations influence training, they are part of $\Phi_t$. The final
checkpoint is $\theta_{T+1}=\Phi_{T+1}(\widehat g_{1:T})$.

By (A2), (A3), and a standard randomization (transfer) argument~\cite{kallenberg2021foundations}, each clipped gradient
can be realized as
\begin{equation}
\bar g_{t,i}=\operatorname{clip}_C\bigl(g(\theta_t,x_i,\omega_{t,i})\bigr),
\qquad
\operatorname{clip}_C(v)=v\min\Bigl(1,\tfrac{C}{\lVert v\rVert_2}\Bigr),
\label{eq:app-realize}
\end{equation}
for a fixed measurable map $g$ and randomness $\omega_{t,i}$ that is
independent across slots and rounds and independent of all coins and noise.
The map $g$ includes the fixed rule that replaces a non-finite prompt
gradient by zero.

\subsection{Hockey-stick divergence and domination}
\label{app:domination}

The hockey-stick divergence $H_\alpha$ of Definition~\ref{def:profile} is
the $f$-divergence with convex $f(u)=(u-\alpha)_+$. A mechanism
$\mathcal M$ is $(\epsilon,\delta)$-DP if and only if
$H_{e^\epsilon}(\mathcal M(D)\Vert\mathcal M(D'))\le\delta$ for every ordered
neighboring pair~\cite{balle2018privacy}.

\begin{definition}[Domination~\cite{zhu2022optimal}]
\label{def:dominate}
An ordered pair of measures $(P,Q)$ \emph{dominates} $(M,M')$ if
$H_\alpha(M\Vert M')\le H_\alpha(P\Vert Q)$ for all $\alpha\ge0$.
\end{definition}

Throughout, $A=A_{q,\sigma}$ and $B=B_\sigma$ of~\eqref{eq:AB}, so
Lemma~\ref{lem:round-main} states that $(A,B)$ dominates one round in the
remove direction and $(B,A)$ in the add direction. These are tight
dominating pairs for the Poisson-subsampled Gaussian
mechanism~\cite{balle2018privacy,zhu2022optimal}, which is why
$\delta_{q,\sigma,T}$ of~\eqref{eq:profile} is the quantity the PLD
accountant computes. Requiring domination for all $\alpha\ge0$, not only
$\alpha\ge1$, is what makes it compose (Lemma~\ref{lem:comp}).

\begin{lemma}[Properties of $H_\alpha$]
\label{lem:hs}
For $\alpha\ge0$:
\begin{enumerate}
\item[(i)] \emph{Data processing.} $H_\alpha(KP\Vert KQ)\le H_\alpha(P\Vert Q)$
for every Markov kernel $K$, with equality if $K$ is a bijection with
measurable inverse.
\item[(ii)] \emph{Joint convexity.} For measurable kernels $u\mapsto P_u,Q_u$
and a common mixing law $\mu$,
$H_\alpha\bigl(\int P_u\,d\mu\,\big\Vert\int Q_u\,d\mu\bigr)
\le\int H_\alpha(P_u\Vert Q_u)\,d\mu\le\sup_uH_\alpha(P_u\Vert Q_u)$.
\item[(iii)] \emph{Common factors.}
$H_\alpha(P\otimes R\Vert Q\otimes R)=H_\alpha(P\Vert Q)$.
\end{enumerate}
\end{lemma}
\begin{IEEEproof}
Parts (i) and (ii) hold for every $f$-divergence with convex $f$~\cite{liese2006divergences}. For (iii),
appending an independent draw from $R$ and discarding the second coordinate
are both Markov kernels, so (i) gives both inequalities.
\end{IEEEproof}

\subsection{Proof of Lemma~\ref{lem:round-main}}
\label{app:proof-round}

\begin{lemma}[Conditional decomposition of one round]
\label{lem:decomp}
Assume (A1)--(A3). Let $D'$ be $D$ with occupied slot $k$ removed, and fix
$h_{<t}$. There are random variables $(S,V,c,Z)$, given by the same
measurable functions of the round-$t$ randomness under $D$ and under $D'$
and hence with the same joint law under both, such that $S$, $(c,V)$, and
$Z$ are mutually independent,
$c\sim\mathrm{Bernoulli}(q)$ is independent of $V$, $\lVert V\rVert_2\le C$,
$Z\sim\mathcal N(0,\sigma^2C^2I_d)$, and the conditional laws of $U_t$
given $h_{<t}$ are
\begin{equation}
\mathcal L_D(U_t\mid h_{<t})=\mathcal L(S+cV+Z),\quad
\mathcal L_{D'}(U_t\mid h_{<t})=\mathcal L(S+Z).
\label{eq:decomp}
\end{equation}
\end{lemma}
\begin{IEEEproof}
By (A1), $h_{<t}$ fixes $\theta_t$. Set
$S=\sum_{i\ne k}c_{t,i}b_i\bar g_{t,i}$, $V=\bar g_{t,k}$, $c=c_{t,k}$, and
$Z=z_t$. For $i\ne k$ the record $(b_i,x_i)$ is the same in $D$ and $D'$, and
by~\eqref{eq:app-realize} $\bar g_{t,i}$ is a function of
$(\theta_t,x_i,\omega_{t,i})$; hence $S$ is a function of
$(c_{t,-k},\omega_{t,-k})$ with the same law under both datasets. By (A3)
the round-$t$ coins, rollout randomness, and noise are fresh, mutually
independent, and independent of $h_{<t}$, so $S$, $(c,V)$, and $Z$ are
independent and $c$ is independent of $V$. Clipping gives
$\lVert V\rVert_2\le C$. Under $D'$, slot $k$ is absent and contributes the
zero vector.
\end{IEEEproof}

\begin{IEEEproof}[Proof of Lemma~\ref{lem:round-main}]
Fix $h_{<t}$ and let $D'$ be $D$ with slot $k$ removed. By
Lemma~\ref{lem:decomp}, conditioned on $(S,V)=(s,v)$ the two conditional
laws of $U_t$ are
\begin{align}
P_{s,v}&=(1-q)\,\mathcal N(s,\sigma^2C^2I_d)+q\,\mathcal N(s+v,\sigma^2C^2I_d),\\
Q_s&=\mathcal N(s,\sigma^2C^2I_d).
\end{align}
Let $R$ be orthogonal with $Rv=\lVert v\rVert_2e_1$ and apply the bijection
$y\mapsto R(y-s)/C$; by Lemma~\ref{lem:hs}(i) this preserves $H_\alpha$. The
images are $A_a\otimes\mathcal N(0,\sigma^2I_{d-1})$ and
$B\otimes\mathcal N(0,\sigma^2I_{d-1})$, where
$A_a=(1-q)\,\mathcal N(0,\sigma^2)+q\,\mathcal N(a,\sigma^2)$ and
$a=\lVert v\rVert_2/C\in[0,1]$. Lemma~\ref{lem:hs}(iii) removes the common
factor. The Markov kernel $K_a(y)=ay+\xi$ with independent
$\xi\sim\mathcal N(0,(1-a^2)\sigma^2)$ maps $\mathcal N(m,\sigma^2)$ to
$\mathcal N(am,\sigma^2)$, hence $A$ to $A_a$ and $B$ to $B$. By
Lemma~\ref{lem:hs}(i),
\begin{equation}
H_\alpha(P_{s,v}\Vert Q_s)\le H_\alpha(A\Vert B),
\qquad
H_\alpha(Q_s\Vert P_{s,v})\le H_\alpha(B\Vert A).
\end{equation}
Finally, $\mathcal L_D(U_t\mid h_{<t})=\int P_{s,v}\,d\mu$ and
$\mathcal L_{D'}(U_t\mid h_{<t})=\int Q_s\,d\mu$ for the common law $\mu$
of $(S,V)$, so Lemma~\ref{lem:hs}(ii) gives domination of the pair of
conditional laws of $U_t$ by $(A,B)$ and of the reversed pair by $(B,A)$.
Applying the bijection $U_t\mapsto\widehat g_t=U_t/(qn)$
(Lemma~\ref{lem:hs}(i)) transfers both bounds to
$(\mathcal M_t(D\mid h_{<t}),\mathcal M_t(D'\mid h_{<t}))$, which
is~\eqref{eq:round}. Nothing in the argument depends on $t$, $h_{<t}$,
$G$, or $C$.
\end{IEEEproof}

\subsection{Composition step of Theorem~\ref{thm:privacy}}
\label{app:proof-thm}

The proof of Theorem~\ref{thm:privacy} in Section~\ref{sec:privacy-guarantee}
invokes one external result, stated here in the notation of this appendix.

\begin{lemma}[Adaptive composition~\cite{zhu2022optimal}]
\label{lem:comp}
Fix an ordered pair $(D,D')$ and measurable kernels
$h_{<t}\mapsto\mathcal M_t(\cdot\mid h_{<t})$. If for every $t$ and every
$h_{<t}$ the pair $(P_t,Q_t)$ dominates
$(\mathcal M_t(D\mid h_{<t}),\mathcal M_t(D'\mid h_{<t}))$, then
$(\bigotimes_tP_t,\bigotimes_tQ_t)$ dominates
$(\mathcal L_D(\widehat g_{1:T}),\mathcal L_{D'}(\widehat g_{1:T}))$.
\end{lemma}
The tradeoff-function composition theorem of Gaussian differential
privacy~\cite{dong2022gaussian} gives the same conclusion. Applied with
$(P_t,Q_t)=(A,B)$ for the remove direction and $(B,A)$ for the add
direction, which Lemma~\ref{lem:round-main} supplies uniformly in $h_{<t}$,
it yields
\begin{align}
H_{e^\epsilon}\bigl(\mathcal L_D(\widehat g_{1:T})\,\big\Vert\,\mathcal L_{D'}(\widehat g_{1:T})\bigr)
&\le H_{e^\epsilon}(A^{\otimes T}\Vert B^{\otimes T}),\\
H_{e^\epsilon}\bigl(\mathcal L_{D'}(\widehat g_{1:T})\,\big\Vert\,\mathcal L_D(\widehat g_{1:T})\bigr)
&\le H_{e^\epsilon}(B^{\otimes T}\Vert A^{\otimes T}),
\end{align}
both at most $\delta_{q,\sigma,T}(\epsilon)$ of~\eqref{eq:profile}; this
is the step used in the proof of Theorem~\ref{thm:privacy}.

\subsection{Corollaries}
\label{app:corollaries}

\begin{corollary}[Replacing a prompt]
\label{cor:replace}
Under the assumptions of Theorem~\ref{thm:privacy}, let
$\delta=\delta_{q,\sigma,T}(\epsilon)$. The output
$(\theta_2,\dots,\theta_{T+1})$ satisfies~\eqref{eq:dp} with
$(2\epsilon,\,(1+e^{\epsilon})\delta)$ for every pair of prompt datasets
that differ only in the record at one occupied slot.
\end{corollary}
\begin{IEEEproof}
Let $D$ and $D'$ differ by the prompt at one occupied slot, and let $D_0$ be
either dataset with that slot absent. $D\simeq D_0\simeq D'$, so applying
Theorem~\ref{thm:privacy} twice gives, for every output set $O$,
$\Pr[O\mid D]\le e^{\epsilon}\Pr[O\mid D_0]+\delta
\le e^{2\epsilon}\Pr[O\mid D']+(1+e^{\epsilon})\delta$. Accounting directly
under the replace-one relation can be tighter.
\end{IEEEproof}

\begin{corollary}[Per-record variants of lines 4--6]
\label{cor:baselines}
Let Algorithm~$1'$ be Algorithm~\ref{alg:dpgrpo} with lines~4--6 replaced
by any computation that, given $\theta_t$, maps the record $x_i$ and fresh
private randomness to a gradient $g_i$. Under (A1)--(A4),
Algorithm~$1'$ satisfies Theorem~\ref{thm:privacy} with the same
$(q,\sigma,T)$. In particular this covers verifier-filtered DP-SFT, which
samples $G$ traces from the public $\pi_{\rm ref}$, keeps those the verifier
accepts, and sets $g_i$ to the gradient of their mean token-normalized
log-likelihood (zero if none is accepted), and gold-SFT, which sets $g_i$
to the gradient of the log-likelihood of the reference solution and takes
the problem together with that solution as the record.
\end{corollary}
\begin{IEEEproof}
The realization~\eqref{eq:app-realize} places no restriction on the map $g$
beyond measurability, and Lemmas~\ref{lem:decomp}, \ref{lem:round-main},
and~\ref{lem:comp} use lines~4--6 only through it. Both comparators keep
lines~2--3 and 7--9, so~\eqref{eq:app-round} and~\eqref{eq:app-realize}
describe them with a different $g$, and the proof of
Theorem~\ref{thm:privacy} goes through unchanged.
\end{IEEEproof}

\begin{remark}[Repeated releases]
The guarantee is per run. Releasing several runs trained on the same records
composes their transcripts, for example by composing their privacy-loss
distributions.
\end{remark}

\subsection{Numerical accounting of every configuration}
\label{app:rdp}

Table~\ref{tab:rdp} lists every $(n,m,\sigma,T)$ used in this paper with
$\delta=1/(10n)$, the certified PLD budget, and the corresponding
R\'enyi-DP (RDP) budget. The PLD column is the \texttt{dp\_accounting}
0.4.3 accountant with the add-or-remove relation, which composes the
privacy-loss distributions of $(A,B)$ and $(B,A)$ and reports the larger
$\epsilon$; its default pessimistic connect-the-dots discretization
upper-bounds the exact hockey-stick curve~\cite{doroshenko2022connect}.
Since $\delta_{q,\sigma,T}$ is nonincreasing in $\epsilon$, a returned value
$\epsilon^\ast$ certifies $(\epsilon,\delta)$-DP for every
$\epsilon\ge\epsilon^\ast$ by Theorem~\ref{thm:privacy}. Configurations
that share $(n,m,\sigma,T)$ share both budgets regardless of loss, clipping
norm, learning rate, or verifier (Corollary~\ref{cor:baselines}); in
particular the DP-SFT and gold-SFT comparators are accounted identically to
\ours{}.

The RDP column composes the RDP of the Poisson-subsampled Gaussian
mechanism, evaluated exactly at integer orders and by a convergent series
at fractional orders~\cite{mironov2019sampled},
over $T$ rounds and converts to $(\epsilon,\delta)$ with the improved
conversion of Balle et al.~\cite{balle2020hypothesis}, minimized over the
default order grid of \texttt{dp\_accounting} 0.4.3
($\{1.1,1.2,\dots,10.9\}\cup\{11,\dots,63\}\cup\{128,256,512,1024\}$).
Opacus~\cite{yousefpour2021opacus} applies the same conversion over a
slightly smaller default grid ($\{1.1,\dots,10.9\}\cup\{12,\dots,63\}$)
and returns the same values as Table~\ref{tab:rdp} to four decimal places. The two
accountants bound the same
composition; the PLD value is the tighter one, and the gap comes almost
entirely from the RDP-to-$(\epsilon,\delta)$ conversion, which is loose at
large $\epsilon$ and small
$\sigma$~\cite{balle2020hypothesis,gopi2021numerical,zhu2022optimal}. A
reader who prefers RDP should read the $\epsilon\in\{1,4.23,8\}$ frontier
as $\epsilon\in\{1.46,5.32,9.83\}$.

\begin{table}[t]
\caption{Every accounting configuration in this paper, with
$\delta=1/(10n)$ and $q=m/n$, under PLD (reported) and RDP accounting.
Parenthesized: the minimizing RDP order.}
\label{tab:rdp}
\centering
\footnotesize
\setlength{\tabcolsep}{3pt}
\begin{tabular}{@{}lcccc@{}}
\toprule
Configuration & $n$ & $(m,\sigma,T)$ & PLD & RDP\\
\midrule
MATH, $\epsilon{=}1$ arm & 6,988 & (128, 1.002, 50) & 1.000 & 1.463 (7.6)\\
MATH, $\epsilon{=}4.23$ arm & 6,988 & (128, 0.612, 50) & 4.235 & 5.322 (3.2)\\
MATH, reference$^{\dagger}$ & 6,988 & (128, 0.484, 50) & 8.000 & 9.830 (2.3)\\
MATH, DP-SFT $T{=}100$ & 6,988 & (128, 0.519, 100) & 8.000 & 9.702 (2.3)\\
MATH, coverage & 6,988 & (512, 0.702, 50) & 8.000 & 9.323 (2.6)\\
MATH, hist.\ rank-4 & 6,988 & (128, 0.612, 100) & 5.038 & 6.185 (3.0)\\
GSM8K, CSQA$^{\ddagger}$ & 6,961 & (128, 0.520, 100) & 8.000 & 9.701 (2.3)\\
\bottomrule
\multicolumn{5}{@{}p{0.97\columnwidth}@{}}{\scriptsize $^{\dagger}$Shared by
the zero-advantage control, clipping-norm study, lower-compute arm, both DP-SFT
$T{=}50$ pairs, and gold-SFT. $^{\ddagger}$Every arm on both benchmarks.
MATH, hist.\ rank-4 is the canary run (Appendix~\ref{app:mechanism-verification}).}
\end{tabular}
\end{table}

\subsection{Status of the assumptions}
\label{app:assumptions}

\paragraph{(A1)}
Holds for every declared run: each run's public inputs are fixed before it
starts, the step count $T$ is fixed in advance, line~8 divides by the
public $qn$ rather than the realized batch size, and constant-reward
groups are retained, so no choice introduces data-dependent control flow.
Hyperparameters were selected on pilot runs over the same training
prompts, following the convention of the private fine-tuning
literature~\cite{li2022large,yu2022differentially}, and, as in that
literature, the stated budget covers the reported run and not the
selection; a private tuning analysis~\cite{papernot2022hyperparameter}
would add a bounded amount. One supervised comparator, the
resumed 100-update DP-SFT run, had its verifier-trip limit raised after a
trip on private completions (Appendix~\ref{app:sft-recovery}); that
choice is the one run-time decision made with access to $D$, and it
concerns a comparator, not \ours{}.

\paragraph{(A2)}
Enforced on the training side: each prompt group is backpropagated
separately (line~6), clipping follows chunk accumulation (line~7), and a
non-finite prompt gradient is replaced by zero by a fixed per-record rule.
The cross-prompt isolation and chunk-identity tests in
Appendix~\ref{app:mechanism-verification} check this property. A non-finite
\emph{noisy} update skips the optimizer step, which is post-processing. Two
execution effects can break (A2) outside exact, isolated computation, and the
theorem does not cover them. First, the rollouts of line~4 for all sampled
prompts are submitted as independent single-sequence requests to one
inference engine. In exact arithmetic, batch composition does not change any
request's conditional token distribution; floating-point inference kernels
are not guaranteed to be batch-invariant, so a prompt's realized rollouts may
depend on which prompts share its batch. Second, the bounded training
verifier of line~5 applies a 10-second wall-clock limit, so a reward can
depend on machine load created by other records. We measure neither effect.
Verifier-filtered DP-SFT (Corollary~\ref{cor:baselines}) generates its
traces through the same engine and filters them with the same verifier, so
both gaps apply to it; gold-SFT generates nothing and calls no verifier, so
neither applies.

\paragraph{(A3)}
Holds for every released model. Their sampling and noise generators are
seeded from operating-system entropy, and recovery files, including
generator states and verifier trip logs that contain private completions,
are not released. The earlier design-study runs seeded their generators
from documented experiment seeds for reproducibility
(Appendix~\ref{app:config-shared}); they are utility measurements and are
not among the released models.

\paragraph{(A4)}
Idealized. Floating-point noise generation can violate DP guarantees that
hold over the reals~\cite{mironov2012significance,jin2022are}, and the
measured clipped sensitivity exceeds $C$ by one unit in the last place.
Data-dependent system telemetry, such as per-step overflow counts, is outside
the release covered here. The theorem certifies the mechanism, not its
finite-precision execution.

\section{Mechanism Verification}
\label{app:mechanism-verification}

The tests in this appendix verify that the implementation realizes the stated
mechanism. They do not replace the differential-privacy proof or accountant.
The machine-readable record is
\texttt{results/mechanism\_verification.json}; it records Python, PyTorch,
\texttt{dp-accounting}, Opacus, and NumPy versions with a UTC timestamp.

\paragraph{Global add/remove sensitivity.}
We form a base dataset of five prompt gradients for a two-layer double-precision
model. For each of 20 random additions and one adversarial addition with
magnitude $10^{12}$, we recompute the sum of independently clipped prompt
gradients. With $C=0.7$, the largest observed
$\lVert f(D')-f(D)\rVert_2/C$ over 21 trials exceeds one by
$2.2\times10^{-16}$, one binary64 unit in the last place, below the test
tolerance $10^{-10}$.

\paragraph{Non-private equivalence.}
At $q=1$, $\sigma=0$, and $C=\infty$, the trainer averages $n$ prompt losses,
each of which averages exactly $G$ rollout losses. For equal group size, this
is the flat mean over all $nG$ sequences. A double-precision SGD step through
the private trainer and a flat baseline has maximum absolute parameter
difference zero at tolerance $10^{-12}$.

\paragraph{Poisson frequency and independence.}
For six public slots sampled independently with $q=0.2$ over 20,000 trials, the
maximum absolute inclusion-frequency error is $0.0040$ and the maximum absolute
off-diagonal inclusion correlation is $0.0141$. These are below
the preregistered tolerances $0.012$ and $0.035$, respectively. This CPU gate is
separate from the rollout-generation regression test, which requires each
completion to be an unseeded single-sequence request.

\paragraph{Cross-prompt isolation.}
The test computes prompt $i$'s gradient, evaluates prompt $j$, changes all of
prompt $j$'s rewards, evaluates it again, and then recomputes prompt $i$.
Prompt $i$'s gradient is bit-identical before and after the perturbation. A
second test exercises this property through the production adapter and trainer,
so batching cannot move reward statistics across prompt boundaries.

\paragraph{Zero-advantage groups.}
All-correct and all-wrong groups produce exactly zero normalized advantages.
Both losses and gradients remain finite; the two gradients are bit-identical
and nonzero because the KL term remains active. Such records are retained and
clipped rather than removed through data-dependent control flow.

\paragraph{External accountant validation.}
For $q=0.1$, noise multiplier $3.0$, 512 steps, and $\delta=10^{-6}$, the
repository privacy-loss-distribution (PLD) accountant returns
$\epsilon=3.7697$, against the published JAX-Privacy reference of $3.77$
(absolute difference $3.1\times10^{-4}$). Opacus RDP returns the larger
value $4.052$ (order 7).

\paragraph{Rollout-chunked backward identity.}
For eight rollouts, the test compares one unchunked prompt backward pass with
four chunked passes of size two accumulated into the same prompt gradient.
The detached scalar losses agree at $(\mathrm{rtol},\mathrm{atol})=(10^{-6},
10^{-7})$, and every parameter gradient agrees at $(10^{-10},10^{-12})$.
The production-adapter test repeats the comparison for chunk sizes one, two,
four, and eight. Clipping occurs only after all chunks have accumulated, so
chunking changes memory geometry but not the prompt record or clipping boundary.

\paragraph{Separate fp16 preflight.}
At each of the three clipping norms, 21 sensitivity trials give a maximum
ratio one fp32 ULP ($1.2\times10^{-7}$) above one. Cross-prompt gradients
are bit-identical and constant-reward groups are retained. The designated
chunk-identity case has maximum absolute difference
$2^{-12}$, so it is not bit-identical. A follow-up set records no
clip-decision flips at the three tested norms; neither test bounds the
discrepancy near every clipping threshold.

\paragraph{Reproduction entry points.}
\path{scripts/emit_verification_artifact.py} produces the machine-readable
record from the suites \path{tests/test_dp_trainer.py},
\path{tests/test_chunked_grpo.py} and \path{tests/test_math_dp_adapter.py};
the fp16 records are kept separately from the binary64 ones.

\paragraph{Canary diagnostic.}
One 100-update bf16 run in the canary configuration of
Table~\ref{tab:full-params} (nominal $\epsilon=5.04$ at
$\delta=1/69{,}880$) replaced 1,000 of the 6,988 training slots with
randomized canaries, 532 present and 468 null. With $q=128/6{,}988$, a
present canary is sampled $1.83$ times in expectation and never with
probability $(1-q)^{100}=0.157$; realized inclusions were not logged. The
empirical lower bound returned is $0.0$. The audit therefore gives no
positive lower bound and says nothing about the tightness of the
accountant; its generators were seeded for reproducibility, so it is a
mechanism diagnostic rather than one of the released models
(Appendix~\ref{app:config-shared}). A powered audit
is future work (Section~\ref{sec:limitations}).

\section{Configurations and Provenance}
\label{app:reproducibility}
\label{app:config-ledger}

\subsection{Shared Training Configuration}
\label{app:config-shared}

Every arm trains rank-4 LoRA adapters ($\alpha{=}8$, dropout $0.05$) on the
seven attention and MLP projection matrices of Qwen2.5-1.5B-Instruct,
$p=4{,}616{,}192$ trainable coordinates in 392 tensors, with
AdamW; the noise $z_t/(qn)$ of Algorithm~\ref{alg:dpgrpo} has expected
norm $\sigma C\sqrt{p}/(qn)\approx0.45$ on MATH and $0.48$ on GSM8K,
eight to nine times the largest possible averaged clipped signal
$C=0.055$. Training runs on one RTX~6000 GPU with fp16 weights and fp32 loss, clipping,
noise and optimizer arithmetic. GRPO completions are drawn as independent,
unseeded single-sequence requests. Poisson sampling and Gaussian noise draw
their randomness from operating-system entropy, independently of the
initialization seed. The earlier design-study runs (the budget sweep of
Table~\ref{tab:design}, the non-private references in
Tables~\ref{tab:robust} and~\ref{tab:gsm8k}, and the canary run of
Appendix~\ref{app:mechanism-verification}) were trained in bf16 on
different hardware and kernels with generators seeded for
reproducibility; each is compared only with the base evaluated under its
own precision, and the two families are never pooled.

Training uses the GRPO trainer of TRL version 0.15.2~\cite{vonwerra2020trl}
with one optimizer step per rollout batch. The loss~\eqref{eq:loss}, the
advantage normalization~\eqref{eq:adv} with $c=10^{-4}$ and a $G-1$
standard-deviation denominator, and the per-token KL
estimator~\eqref{eq:grpo} reproduce that trainer's computation; the
per-prompt clipping, noise, and $qn$ normalization of
Algorithm~\ref{alg:dpgrpo} replace its batch-level gradient step.

\subsection{Mechanism Parameters}
\label{app:config-params}

Table~\ref{tab:full-params} gives every mechanism configuration at full
precision; the body rounds these to three decimals. The sampling rate is
$q=m/n$ exactly and $\delta=1/(10n)$, i.e.\ $1/69{,}880$ on MATH and
$1/69{,}610$ on GSM8K and CommonsenseQA. The clipping norm is
$C=0.05536$ for \ours{} and the zero-advantage control, with
$C\in\{0.02768,0.11072\}$ in the clipping sweep, and $C=0.1143670583$ for
both supervised arms, inherited from the private fine-tuning recipe they
follow~\cite{li2022large,yu2022differentially} and not swept; their
learning rates are $5\times10^{-6}$ (filtered) and $2\times10^{-5}$
(gold), selected in Appendix~\ref{app:config-pilots}.
Table~\ref{tab:rdp} gives the resulting budgets.

\begin{table}[t]
\centering
\caption{Mechanism parameters at full precision. $m$ is the expected
number of prompts per round.}
\label{tab:full-params}
\footnotesize
\setlength{\tabcolsep}{3pt}
\begin{tabular}{@{}lrrrl@{}}
\toprule
Configuration & $n$ & $m$ & $T$ & $\sigma$ \\
\midrule
MATH, reference$^{\dagger}$     & 6{,}988 & 128 & 50  & 0.4842404738068581 \\
MATH, $\epsilon{=}4.23$ (bf16)  & 6{,}988 & 128 & 50  & 0.612058477104 \\
MATH, $\epsilon{=}1$ (bf16)     & 6{,}988 & 128 & 50  & 1.0023417361080647 \\
MATH, coverage ($G{=}2$)        & 6{,}988 & 512 & 50  & 0.7022244453430175 \\
MATH, DP-SFT, 100 updates       & 6{,}988 & 128 & 100 & 0.519267738191411 \\
MATH, canary (bf16)             & 6{,}988 & 128 & 100 & 0.612058477104 \\
GSM8K and CSQA, all arms        & 6{,}961 & 128 & 100 & 0.5196564674377442 \\
\bottomrule
\multicolumn{5}{@{}p{0.97\columnwidth}@{}}{\scriptsize $^{\dagger}$Shared
by \ours{} and the control, the $\epsilon{=}8$ budget-sweep cell, the
clipping sweep, the lower-compute arm ($G{=}2$), and the $T{=}50$ DP-SFT
arms.}
\end{tabular}
\end{table}

\subsection{Evaluation Sets and Scoring}
\label{app:config-eval}

\paragraph{MATH.}
The 4,499-problem endpoint is the 5,000-problem test split minus one
problem duplicated in the training set and the 500 MATH-500 problems that
earlier runs, and the supervised learning-rate pilots
(Appendix~\ref{app:config-pilots}), used for selection. The exclusion is
by problem identity from saved 4,999-problem predictions; no model is
rerun. It was fixed after those earlier runs. The base model solves 2,207 problems (49.1\%)
under fp16 and 2,212 (49.2\%) under bf16; each study is compared with the
base of its own precision, and neither correctness vector is substituted
for the other.

\paragraph{GSM8K, CommonsenseQA and SVAMP.}
GSM8K is scored on its 1,319 test problems. The 6,961 CommonsenseQA
training questions are drawn uniformly at random from the training split
after excluding 700 questions used in a pilot, and all 1,221 validation
questions are scored. SVAMP is scored with GSM8K's prompt, decoding and
cap (Appendix~\ref{app:svamp}).

\paragraph{Statistics.}
Problem-cluster bootstrap intervals use 100,000 draws and
sign-randomization $p$-values use 1,000,000 draws, each from a fixed seed
recorded with the analysis. On CommonsenseQA, Holm adjustment runs across
its eight contrasts (four comparators under two scorings).

\subsection{Learning-Rate Pilots}
\label{app:config-pilots}

Each supervised recipe's learning rate was chosen by one pilot run per
value in $\{5\times10^{-6},2\times10^{-5},8\times10^{-5},3\times10^{-4}\}$
on MATH, with one initialization seed and the private configuration of
Table~\ref{tab:full-params}, selected by MATH-500 accuracy at step 50 with
ties broken toward the smaller rate. MATH-500 is disjoint from the 4,499-problem endpoint. Filtered DP-SFT selected
$5\times10^{-6}$ and gold DP-SFT $2\times10^{-5}$; both were carried
unchanged to GSM8K and CommonsenseQA. In increasing order of rate, the
filtered pilots scored 50.6, 49.2, 49.6 and 50.6\% on MATH-500, a flat
profile whose selection is a tie, and the gold pilots scored 47.8, 48.2,
40.6 and 29.4\%. The public base answers 48.6\% of the same problems under
the endpoint evaluation, so no pilot rate improves on it by more than two
points. Training loss is private telemetry and was not logged
(Section~\ref{sec:threat}), so these pilots are the only evidence on the
supervised learning rates. As is conventional in private fine-tuning,
the pilots' use of the training prompts is not charged to the reported
budget (Appendix~\ref{app:assumptions}).

\subsection{Resumed and Recovered Runs}
\label{app:config-recovery}

Six reported trajectories were resumed from checkpoints. Each checkpoint
stores the optimizer, privacy ledger, and framework, CUDA, sampling and
noise generator states, which are restored after inference-engine
initialization. None is claimed to be bitwise equivalent to uninterrupted
execution, and the recovery states are not part of any release.

\paragraph{MATH clipping sweep, high-$C$ seed 2.}
The $C=0.11072$ seed-2 run resumed from checkpoint 35, replaying updates
36--40 and then continuing to 50: 55 optimizer computations, 50
compositions in its final ledger. The inference engine reseeded framework
and CUDA randomness after the dedicated generators were restored.
Table~\ref{tab:design} reports this trajectory, not its clean replication.
The ledger accounts for the retained trajectory only; a release that
included the discarded branch as well would have to compose both.

\paragraph{MATH 100-update DP-SFT, seed 1.}
\label{app:sft-recovery}
The run exceeded the training-verifier trip limit, which was raised from
5 to 10, and resumed from checkpoint 45, replaying updates 46--50: 105
optimizer computations produced the retained 100-update trajectory. The
accountant covers the retained trajectory, not the abandoned branch. The
trip occurred on private completions and the limit was raised afterwards,
so this comparator run depends on one choice made with access to $D$
(Appendix~\ref{app:assumptions}).

\paragraph{CommonsenseQA, seeds 1--2 of \ours{} and the control.}
The four runs paused cleanly before scheduled cluster maintenance, at
checkpoints 60--66, and resumed from those checkpoints. Each saved step
equalled the number of completed updates, so no update was replayed.

\paragraph{Allocation study.}
Three allocation runs failed during process shutdown after training had
finished. Their final checkpoints were admitted after independent
verification, without further training.

\subsection{Protocol Timeline}
\label{app:protocol-timeline}

Table~\ref{tab:protocol-timeline} collects every decision that was made
after some data or result had been seen, with the runs it preceded and
followed; each is discussed where it arises in the text.

\begin{table}[t]
\centering
\caption{Protocol decisions and their timing.}
\label{tab:protocol-timeline}
\footnotesize
\setlength{\tabcolsep}{3pt}
\begin{tabular}{@{}p{0.36\columnwidth}p{0.6\columnwidth}@{}}
\toprule
Decision & Fixed when \\
\midrule
Strict scorers (all benchmarks) & Before any fixed-protocol arm was run (Sec.~\ref{sec:design-endpoints}). \\
\ours{} LR, $C$, $G$, $T$ on MATH & In the MATH design study, before the fixed protocol (Sec.~\ref{sec:design-stage1}). \\
Supervised learning rates & Four-point pilots on MATH-500, before the fixed protocol (App.~\ref{app:config-pilots}). \\
GSM8K horizon and cap & Before any private GSM8K run, from one earlier non-private run (Sec.~\ref{sec:design-setup}). \\
MATH-500 excluded from evaluation & After the earlier runs that used it for selection; applied by problem identity to saved predictions (App.~\ref{app:config-eval}). \\
Seeds 3--4, \ours{} and control, MATH and GSM8K & After the seed 1--2 results were inspected (Sec.~\ref{sec:design-arms}); four-seed contrasts are sequential. \\
CSQA answer declaration format & After a pilot with a stricter marker failed, before the CSQA arms (Sec.~\ref{sec:design-setup}). \\
CSQA final-answer rescore rule & After strict results of base, \ours{}, control and gold SFT were known; hash-frozen before rescoring them and before filtered-SFT evaluation (App.~\ref{app:csqa-extraction}). \\
MATH format-robust rules & Developed on public-base predictions; two arm-blind audits before arm-level rescoring (App.~\ref{app:format-robust}). \\
Verifier trip limit $5\to10$ & After one 100-update DP-SFT run tripped; that run resumed (App.~\ref{app:sft-recovery}). \\
Resumed trajectories & High-$C$ seed 2 (updates 36--40 replayed); 100-update DP-SFT seed 1 (46--50 replayed); CSQA seeds 1--2 of \ours{} and control (no replay); three allocation runs admitted at their final checkpoint (App.~\ref{app:config-recovery}). \\
\bottomrule
\end{tabular}
\end{table}

\section{Supplementary Results}
\label{app:historical-results}

This appendix gives per-seed results and the supporting design studies.
Each contrast is computed within its own study; historical bf16 runs and
the fp16 comparisons are kept separate.

\begin{table}[t]
\centering
\caption{Every arm, the benchmarks it ran on, its seeds, and where its
results are reported. Seeds for Stage~2 are given as MATH / GSM8K / CSQA.}
\label{tab:arms}
\footnotesize
\setlength{\tabcolsep}{2.5pt}
\begin{tabular}{@{}llll@{}}
\toprule
Arm & Benchmark & Seeds & Reported \\
\midrule
\multicolumn{4}{@{}l}{\emph{Stage 1: design study}} \\
$\epsilon\in\{1,4.23,8\}$ (bf16)          & MATH  & 2 each & Tab.~\ref{tab:design} \\
$C\in\{\tfrac12,1,2\}\times0.055$         & MATH  & 2 each & Tab.~\ref{tab:design} \\
Rollout allocation (2 alt.)               & MATH  & 2 each & Tab.~\ref{tab:design} \\
Same-LR and 100-update SFT                & MATH  & 2 each & Sec.~\ref{sec:res-c2} \\
$\sigma{=}0$, clipped / unclipped         & GSM8K & 2 each & Fig.~\ref{fig:controls} \\
$\sigma{=}0$, clipped (bf16)              & MATH  & 2      & App.~\ref{sec:res-c3b} \\
$\sigma{=}0$, unclipped, LR $10^{-4}$     & GSM8K & 1      & App.~\ref{sec:res-c3b} \\
\midrule
\multicolumn{4}{@{}l}{\emph{Stage 2: fixed protocol}} \\
Public base                  & all three & --        & Fig.~\ref{fig:results}, \\
\ours{} ($\epsilon{=}8$)    & all three & 4 / 4 / 2 & \quad Tabs.~\ref{tab:contrasts}, \\
Zero-advantage control       & all three & 4 / 4 / 2 & \quad \ref{tab:robust}, \ref{tab:csqa} \\
Filtered DP-SFT              & all three & 2 / 2 / 2 & \\
Gold DP-SFT                  & all three & 2 / 2 / 2 & \\
\midrule
\multicolumn{4}{@{}l}{\emph{Evaluation only}} \\
GSM8K checkpoints, all arms  & SVAMP & as GSM8K & App.~\ref{app:svamp} \\
\midrule
\multicolumn{4}{@{}l}{\emph{Historical non-private runs (bf16, rank 16)}} \\
Standard GRPO, 900 updates    & MATH  & 1 & Tab.~\ref{tab:robust} \\
Standard GRPO, 2 epochs       & GSM8K & 1 & Tab.~\ref{tab:gsm8k} \\
\bottomrule
\end{tabular}
\end{table}

\subsection{Per-Seed Results of the Fixed Protocol}
\label{app:per-seed}

\begin{table}[h]
\centering
\caption{Strict accuracy (\%) of every seed plotted in
Fig.~\ref{fig:results}, and the MATH format-robust accuracies behind
Table~\ref{tab:robust}. Seeds in run order.}
\label{tab:per-seed}
\footnotesize
\setlength{\tabcolsep}{4pt}
\begin{tabular}{@{}llr@{}}
\toprule
Arm & Seeds & Mean \\
\midrule
\multicolumn{3}{@{}l}{\emph{MATH, strict}} \\
Public base      & --                        & 49.1 \\
Zero-adv.\ control & 48.9 / 48.8 / 49.1 / 48.8 & 48.9 \\
\ours{}        & 51.3 / 51.9 / 51.8 / 51.3 & 51.6 \\
Filtered DP-SFT  & 49.3 / 49.3               & 49.3 \\
Gold DP-SFT      & 48.7 / 48.6               & 48.6 \\
\midrule
\multicolumn{3}{@{}l}{\emph{GSM8K, strict}} \\
Public base      & --                        & 70.7 \\
Zero-adv.\ control & 70.9 / 71.1 / 71.6 / 70.7 & 71.1 \\
\ours{}        & 74.1 / 73.9 / 74.3 / 75.0 & 74.3 \\
Filtered DP-SFT  & 70.4 / 70.7               & 70.5 \\
Gold DP-SFT      & 70.1 / 70.4               & 70.2 \\
\midrule
\multicolumn{3}{@{}l}{\emph{MATH, format-robust}} \\
Public base      & --                        & 50.4 \\
Zero-adv.\ control & 50.3 / 50.2 / 50.3 / 50.4 & 50.3 \\
\ours{}        & 51.3 / 52.0 / 51.8 / 51.3 & 51.6 \\
\bottomrule
\end{tabular}
\end{table}

\begin{table}[h]
\centering
\caption{Per-seed differences (points) behind Table~\ref{tab:contrasts},
with Holm-adjusted exact McNemar $p$ (Holm within each contrast). The
supervised contrasts pair \ours{} seeds 1--2 with the supervised seeds.}
\label{tab:per-seed-contrasts}
\footnotesize
\setlength{\tabcolsep}{3pt}
\begin{tabular}{@{}lll@{}}
\toprule
Contrast & Per-seed difference & Holm $p$, largest \\
\midrule
\multicolumn{3}{@{}l}{\emph{MATH}} \\
\ours{} $-$ control  & $+2.31$ / $+3.16$ / $+2.62$ / $+2.51$ & $6.4\times10^{-5}$ \\
\ours{} $-$ filtered & $+1.98$ / $+2.67$                     & $2.4\times10^{-4}$ \\
\ours{} $-$ gold     & $+2.60$ / $+3.38$                     & $2.9\times10^{-6}$ \\
\midrule
\multicolumn{3}{@{}l}{\emph{GSM8K}} \\
\ours{} $-$ control  & $+3.26$ / $+2.81$ / $+2.65$ / $+4.25$ & $0.013$ \\
\ours{} $-$ filtered & $+3.79$ / $+3.26$                     & $1.1\times10^{-3}$ \\
\ours{} $-$ gold     & $+4.09$ / $+3.49$                     & $7.5\times10^{-4}$ \\
\bottomrule
\end{tabular}
\par\smallskip
\raggedright\footnotesize
MATH format-robust \ours{} $-$ control per seed: $+0.98$ / $+1.84$ /
$+1.56$ / $+0.98$.
\end{table}

\subsection{Budget, Clipping, and Non-Private References}
\label{sec:res-c3b}
\label{sec:res-budget}
\begin{figure}[t]
\centering
\includegraphics[width=\columnwidth]{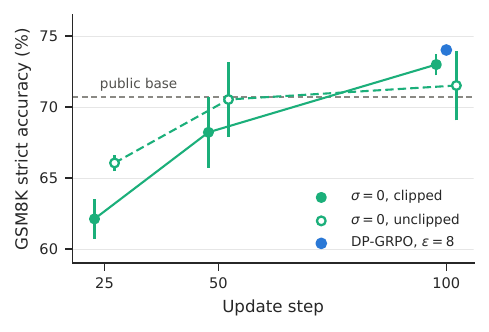}
\caption{Non-private GSM8K references at \ours{}'s hyperparameters:
noise removed, with clipping retained or disabled. Dots show two-seed
means; bars span the seeds, not confidence intervals. \ours{} is the
mean of its two matched seeds at the final step (74.0\%). The earlier
non-private run of Table~\ref{tab:gsm8k} (75.4\%) was trained at a
different rank, learning rate and horizon and is not shown.}
\label{fig:controls}
\end{figure}

Within the tested ranges, neither the accounted budget nor the clipping
norm moves MATH accuracy detectably (Table~\ref{tab:design}). Relative
to $\epsilon{=}8$, the bf16 budget sweep gives $-0.09$ [$-0.66$, $0.48$]
at $\epsilon{=}4.23$ and $-0.57$ [$-1.19$, $0.06$] at $\epsilon{=}1$,
with MDEs 0.82 and 0.89; the interval for an eightfold tighter budget
excludes a cost above 1.19 pp. In the fp16 clipping
sweep, halving $C$ changes accuracy by $-0.30$ [$-0.72$, $0.12$] and
doubling it by $-0.03$ [$-0.37$, $0.30$] (MDEs 0.60 and 0.48). Because
the noise scales with $C$, this sweep moves clipping bias and absolute
noise together and does not separate them.

On the mathematical tasks the $\sigma{=}0$ references at the private
hyperparameters do not exceed \ours{} (Fig.~\ref{fig:controls}). On
GSM8K the matched two \ours{} seeds gain $+3.30$ [$1.44$, $5.16$] over
base; the clipped reference gains $+2.27$ [$0.23$, $4.32$] and the
unclipped one $+0.80$ [$-1.36$, $2.96$], its seeds ending at 74.0\% and
69.1\%. On MATH the two-seed clip-only reference at learning rate
$10^{-4}$ reaches 51.3\%, against 51.6\% for \ours{}, and a single-seed
unclipped GSM8K reference at the same rate reaches 74.3\% by step 25 and
60.1\% at step 100. These references keep hyperparameters selected under
noise and are not a non-private ceiling; the retention ratios of
Section~\ref{sec:res-nonprivate} use the earlier standard GRPO runs, and
on CommonsenseQA, where the reference at the private hyperparameters
exceeds \ours{} by about six points, the comparison is paired.

Table~\ref{tab:design} collects these MATH design comparisons with the
allocation study and the two supervised challenges. The budget rows are
bf16 runs and are compared only with each other. In the SFT block each
change is relative to selected-LR SFT, not to \ours{}. The high-$C$
clipping pair and the 100-update SFT arm each include one resumed seed
(Appendix~\ref{app:config-recovery}).

\begin{table}[t]
\centering
\caption{Supporting MATH design comparisons.}
\label{tab:design}
\footnotesize
\setlength{\tabcolsep}{4pt}
\begin{tabular}{@{}llr@{}}
\toprule
Alternative & Change (pp), 95\% CI & MDE (pp) \\
\midrule
\multicolumn{3}{@{}l}{\emph{Budget; reference: $\epsilon = 8$ (earlier study)}} \\
$\epsilon = 4.23$ & $-0.09$ [$-0.66$, $0.48$] & 0.82 \\
$\epsilon = 1$    & $-0.57$ [$-1.19$, $0.06$] & 0.89 \\
\midrule
\multicolumn{3}{@{}l}{\emph{Clipping; reference: $C = 0.055$}} \\
$C = 0.111$ & $-0.03$ [$-0.37$, $0.30$] & 0.48 \\
$C = 0.028$ & $-0.30$ [$-0.72$, $0.12$] & 0.60 \\
\midrule
\multicolumn{3}{@{}l}{\emph{Rollouts $\times$ prompts; reference: $8 \times 128$}} \\
$2 \times 512$ (same rollouts)    & $-0.98$ [$-1.73$, $-0.23$] & 1.08 \\
$2 \times 128$ (quarter rollouts) & $-0.96$ [$-1.74$, $-0.18$] & 1.12 \\
\midrule
\multicolumn{3}{@{}l}{\emph{SFT; reference: LR $5{\cdot}10^{-6}$, 50 updates}} \\
LR $3{\cdot}10^{-4}$ (\ours{}'s) & $+0.22$ [$-0.70$, $1.14$] & 1.32 \\
100 updates & $+0.54$ [$0.03$, $1.06$] & 0.73 \\
\bottomrule
\end{tabular}
\par\smallskip
\raggedright\footnotesize
Change is alternative minus the reference named in each block.
Two seeds per setting; studies are analyzed separately.
\end{table}

\subsection{Rollouts per Prompt versus Prompts per Round}
\label{sec:res-alloc}

On MATH, using two rather than eight rollouts per prompt reduces measured
training time from 12.1--12.2 hours to 3.7--3.8 hours, a 3.2$\times$
reduction, while lowering strict accuracy by 0.96 [0.18, 1.74] points.
Reallocating the same expected 1,024 rollouts to 512 prompts with two
responses each instead lowers accuracy by 0.98 [0.23, 1.73]
(Table~\ref{tab:design}). The latter changes the sampling rate and its
calibrated noise multiplier as well as group size.

Under format-robust scoring, the reference-minus-alternative gaps shrink
to $+0.49$ [$-0.27$, $1.26$] for lower compute and $+0.62$ [$-0.11$,
$1.37$] for broader coverage. The measured speedup remains, and about
half of the strict-accuracy cost is answer formatting rather than
correctness, consistent with fewer within-group contrasts per prompt
leaving less signal about the output format.

\subsection{Transfer to SVAMP}
\label{app:svamp}
\label{sec:res-svamp}

The released GSM8K checkpoints (step 100, no selection) were scored on
the 1{,}000 problems of SVAMP~\cite{patel2021svamp} with GSM8K's prompt,
decoding and 512-token cap, and a numeric extractor frozen before scoring.
No SVAMP problem matches a GSM8K training question exactly or with token
Jaccard similarity of at least 0.8. Scoring a released model on public
data is post-processing, so this costs no privacy budget.
Table~\ref{tab:svamp} gives the results. The four \ours{} seeds average
83.2\%, versus 82.0\% for base. All four point estimates are higher; the pooled gap is $+1.20$
[$-0.53$, $2.95$], with Holm-adjusted $p=0.575$. With 1,000 problems the
MDE is 2.49 points, so this check can detect a transfer gain of GSM8K's
size (3.6 points) but not one of a third of that size, which is what the
point estimate shows. The seed-level interval is [0.50, 1.90].
Differences from filtered and gold DP-SFT lie within the same MDE. SVAMP
is a transfer check, not an additional training benchmark.
The six contrasts were
Holm-corrected together, including two against the non-private
$\sigma{=}0$ references (77.9\% clipped, 78.1\% unclipped), which are not
in the table: 8.5--12.3\% of their outputs reach the length cap, against
at most 0.8\% for every arm shown, so their SVAMP accuracy reflects
response length rather than transfer.

\begin{table}[h]
\centering
\caption{SVAMP (1,000 problems): GSM8K checkpoints scored without further
training. Strict accuracy per seed, and \ours{}'s difference from each
arm (points, 95\% CI).}
\label{tab:svamp}
\footnotesize
\setlength{\tabcolsep}{3pt}
\begin{tabular}{@{}lllr@{}}
\toprule
Arm & Seeds & Mean & \ours{} $-$ arm \\
\midrule
Public base      & --                        & 82.0 & $+1.20$ [$-0.53$, $2.95$] \\
Zero-adv.\ control & 82.1 / 81.8 / 82.1 / 82.1 & 82.0 & $+1.18$ [$-0.35$, $2.73$] \\
Filtered DP-SFT  & 82.6 / 82.6               & 82.6 & $+0.25$ [$-1.45$, $1.95$] \\
Gold DP-SFT      & 82.2 / 82.7               & 82.5 & $+0.40$ [$-1.40$, $2.20$] \\
\ours{}        & 83.0 / 82.7 / 83.4 / 83.7 & 83.2 & -- \\
\bottomrule
\end{tabular}
\par\smallskip
\raggedright\footnotesize
MDEs 2.49, 2.20, 2.42 and 2.58 points; Holm-adjusted $p \ge 0.58$. The
supervised rows use \ours{} seeds 1--2.
\end{table}

\section{Answer Extraction and Scoring Sensitivity}
\label{app:format-robust}

\subsection{MATH: Strict and Format-Robust Scores}

Strict accuracy is the fraction of all 4,499 evaluation problems with a
correct boxed answer. Format-robust accuracy uses the same denominator
and the same saved responses. It preserves scores for boxed answers and
adds a fallback for unboxed responses shorter than the 1,024-token cap.
Thus a capped response is excluded from the fallback, not from evaluation.

The fallback uses the final explicit answer statement, then a final
inferential statement, then the final nonempty line. Within that segment
it extracts a mathematical expression, the right-hand side of an
assignment, or a stated answer value. It preserves final answers joined
by ``and'' or ``or'' rather than selecting one. The extractor never reads
the gold answer; the extracted candidate is compared with gold by the
existing symbolic equality verifier. The rules were developed on public-base
predictions. Two arm-blind audits led to corrections before unblinding:
excluding capped responses from fallback and preserving multiple answers.
The final deduplicated audit found 60/60 candidates consistent with the
completed final statement and no observed false positives. No verifier
guard tripped across the 47 unique prediction artifacts rescored.

\paragraph{Decomposing the improvement.}
Table~\ref{tab:math-format-decomposition} partitions problems by the
base model's response. ``Questions'' is a count. Each contribution is the
net number of additional correct answers in that group, averaged over
four seeds, divided by \emph{all 4,499 problems} and multiplied by 100.
It is measured in percentage points, not the group's own accuracy.
The contributions therefore sum to the total gain over base, up to rounding.
Problems whose base response has no boxed answer account for $+2.75$
points under strict scoring but $+1.49$ under format-robust scoring. The
total gains
are $+2.51$ [$1.54$, $3.48$] and $+1.26$ [$0.33$, $2.17$], respectively.

\begin{table}[t]
\centering
\caption{Contributions to the MATH gain over base.}
\label{tab:math-format-decomposition}
\footnotesize
\setlength{\tabcolsep}{4pt}
\begin{tabular}{@{}lrrr@{}}
\toprule
Base response & Questions & \multicolumn{2}{c}{Contribution (pp)} \\
\cmidrule(l){3-4}
 & & Strict & Format-robust \\
\midrule
Boxed an answer        & 3,831 & $-0.24$ & $-0.24$ \\
No boxed answer        &   668 & $+2.75$ & $+1.49$ \\
\midrule
All problems           & 4,499 & $+2.51$ & $+1.26$ \\
\bottomrule
\end{tabular}
\end{table}

\paragraph{GSM8K.}
The primary extractor already accepts an unboxed final number. A separate
sensitivity analysis counts every response reaching the length cap as
incorrect; the reward contrast remains $+2.96$ [$1.33$, $4.59$] points.

\subsection{CommonsenseQA: Final-Answer Rescoring}
\label{app:csqa-extraction}

The strict score uses the training reward's letter extractor. Option
discussion can cause that extractor to reject a response even when it
later states a clear answer. The final-answer rule instead reads the last
explicit answer declaration and rejects conflicting choices within that
declaration. Both rules use all 1,221 validation questions and the same
completions. Rescoring changes neither training rewards nor model weights.

The rule was written after the strict results for base, \ours{},
the zero-advantage control, and gold SFT were known. It was hash-frozen before
rescoring those arms and before filtered-SFT evaluation. This is a
post-hoc sensitivity endpoint. The initial audit of 40 changed extractions
found no disagreements and no false positives among 21 rescored-correct
responses. A separate filtered-SFT audit sampled 40 of 578 changed records,
covering 37 distinct questions/completions. It found no disagreements and
no false positives among 23 rescored-correct responses. Seeds, gold
answers, and extracted answers were hidden in that audit, but the
filtered-SFT arm identity was known. These were agent audits, not
independent human assessments; zero observed errors does not establish
an error-free extractor.

The base has 815 correct answers and 21 remaining abstentions. Crediting
every abstention raises it to 836/1,221, still 2.99 points below
\ours{}'s mean of 872.5 correct answers. This descriptive bound addresses
unparsed base answers only, not errors in accepted extractions.

\end{document}